%% file: paper.tex
\documentclass[]{TEAI}
\usepackage{helvet}

\usepackage{amsmath} 
\usepackage{natbib}
\usepackage{graphicx}
\usepackage{subcaption} 

\usepackage[toc,page,header]{appendix}
\usepackage[utf8]{inputenc} % allow utf-8 input
\usepackage[T1]{fontenc}    % use 8-bit T1 fonts
\usepackage{hyperref}       % hyperlinks
\usepackage{url}            % simple URL typesetting
\usepackage{booktabs}       % professional-quality tables
\usepackage{lmodern}        % scalable Latin Modern fonts
\usepackage{amsfonts}       % blackboard math symbols
\usepackage{nicefrac}       % compact symbols for 1/2, etc.
\usepackage{microtype}      % microtypography
\usepackage{wrapfig}
\usepackage{amssymb}  % 用于特殊符号如♠♣等
\usepackage{fontawesome}  % 用于邮箱图标 \faEnvelope
\usepackage{url}  % 用于 \url 命令

\usepackage{titletoc}
\usepackage{algorithm}
\usepackage{tikz}  % 用于绘制彩色标记符号
\usepackage{comment}  % 用于注释备用版本
\usepackage{tabularx}  % 如果需要自动调整列宽
\usepackage{booktabs}  % 用于更美观的表格线条(可选)
\usepackage{minitoc}

\usepackage{booktabs}
\usepackage{array}
\usepackage{etoolbox}

\definecolor{lightblue}{RGB}{200, 230, 255}  
\definecolor{headerblue}{RGB}{150, 200, 255} 

\usepackage{pgfplots}
\usepackage[utf8]{inputenc} % allow utf-8 input
\usepackage[T1]{fontenc}    % use 8-bit T1 fonts
\usepackage{hyperref}       % hyperlinks
\usepackage{url}            % simple URL typesetting
\usepackage{booktabs}       % professional-quality tables
\usepackage{amsfonts}       % blackboard math symbols
\usepackage{nicefrac}       % compact symbols for 1/2, etc.
\usepackage{microtype}      % microtypography
\usepackage{xcolor}         % colors
\usepackage{graphicx}
\usepackage{float}
\usepackage{comment}
\usepackage{multirow} % For multi-row cells
\usepackage{amsmath} % For \text command if needed inside math mode\Delta
\usepackage{makecell} % For multi-line cells and better vertical spacing in cells
\usepackage{siunitx}  % For better number alignment (optional but recommended)
\usepackage{tikz}
\usepackage{pgf-pie} % Package for creating pie charts
\usepackage{subcaption}
\usepackage{wrapfig}
\usepackage[export]{adjustbox}

\usepackage{ragged2e}      % for \RaggedRight in tabularx
\usepackage{tabularx}       % For tables with fixed total width and auto-adjusting columns
\usepackage{array}          % For advanced column formatting (like >{\centering\arraybackslash}X)
\usepackage{caption}        % Recommended for figures/tables, but we'll do simple text below images here.
\usepackage{enumitem}
\usepackage{pifont}
\usepackage[hang,flushmargin]{footmisc} % 更好的脚注处理
\usepackage{algorithmic}
\usepackage{tcolorbox}
\usepackage{amsthm}
\newtheorem{theorem}{Theorem} 
\newtheorem{proposition}[theorem]{Proposition}
\usepackage{tcolorbox}    % 1. 先加载 tcolorbox 主包
\tcbuselibrary{breakable}  % 2. 显式加载 breakable 库
\tcbuselibrary{skins}      % 3. 显式加载 skins 库 (这个库提供 topruleatbreak 等选项)
\usepackage{tabularx}
\usepackage{listings}

\title{Preference Score Distillation: Leveraging 2D Rewards to Align Text-to-3D Generation with Human Preference}

\author{
    Jiaqi Leng\textsuperscript{1}, Shuyuan Tu\textsuperscript{1}, Haidong Cao\textsuperscript{1}, Sicheng Xie\textsuperscript{1}, Daoguo Dong\textsuperscript{1}, Zuxuan Wu\textsuperscript{1,$\dagger$}, Yu-Gang Jiang\textsuperscript{1} 
}

\affiliation[1]{\mbox{Fudan University}} 
\begin{document}

\settitlefigure{%
    \includegraphics[width=0.92\textwidth]{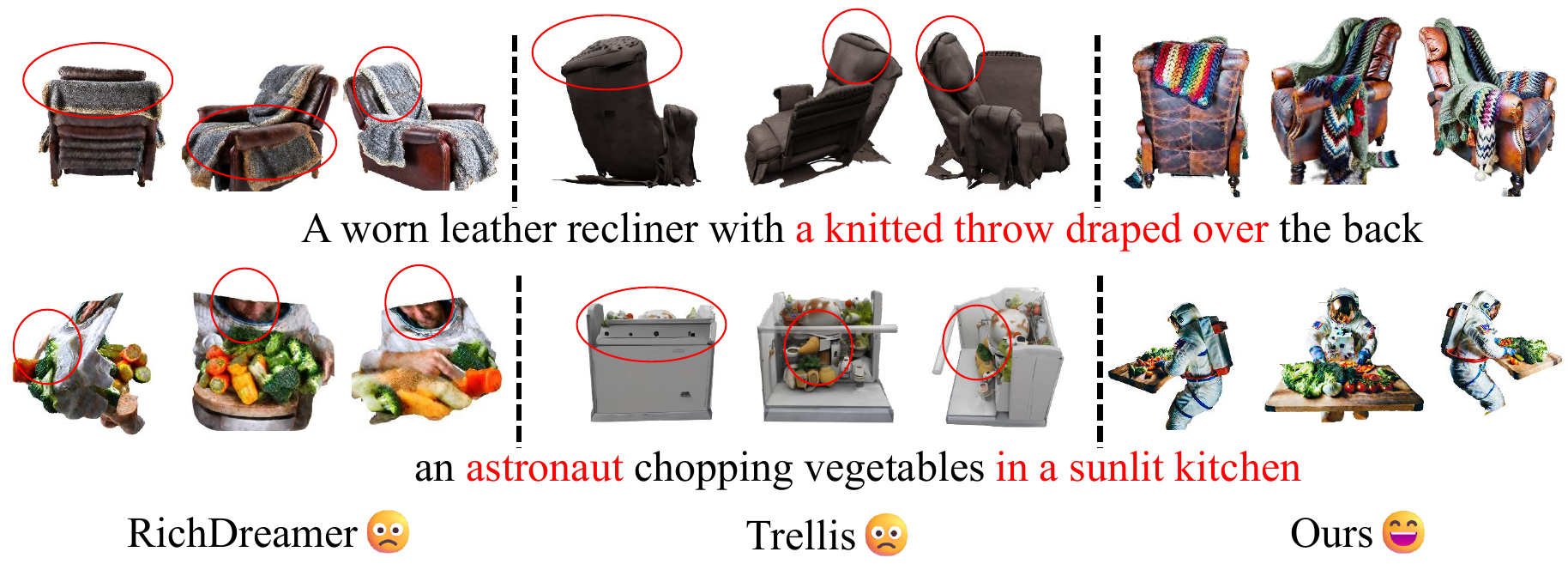}%
    \captionof{figure}{Comparisons with state-of-the-art methods RichDreamer \citep{richdreamer} and Trellis \citep{trellis}. Even when compared to methods that leverage stronger 3D priors, our method achieves significantly higher text alignment and enhanced visual quality, highlighting the critical role of preference alignment in text-to-3D generation.}%
    \label{first}%
}

\abstract{

\begin{abstract}

Human preference alignment presents a critical yet underexplored challenge for diffusion models in text-to-3D generation. Existing solutions typically require task-specific fine-tuning, posing significant hurdles in data-scarce 3D domains. To address this, we propose Preference Score Distillation (PSD), an optimization-based framework that leverages pretrained 2D reward models for human-aligned text-to-3D synthesis without 3D training data. Our key insight stems from the incompatibility of pixel-level gradients: due to the absence of noisy samples during reward model training, direct application of 2D reward gradients disturbs the denoising process. 
Noticing that similar issue occurs in the naive classifier guidance in conditioned diffusion models, we fundamentally rethink preference alignment as a classifier-free guidance (CFG)-style mechanism through our implicit reward model. Furthermore, recognizing that frozen pretrained diffusion models constrain performance, we introduce an adaptive strategy to co-optimize preference scores and negative text embeddings. By incorporating CFG during optimization, online refinement of negative text embeddings dynamically enhances alignment. To our knowledge, we are the first to bridge human preference alignment with CFG theory under score distillation framework. Experiments demonstrate the superiority of PSD in aesthetic metrics, seamless integration with diverse pipelines, and strong extensibility.

\end{abstract}
}
\maketitle
\renewcommand{\thefootnote}{}
\footnotetext{$^\dagger$Corresponding authors.}
\renewcommand{\thefootnote}{\arabic{footnote}}

% Catalogue (Need \newpage)
% \newpage
% \tableofcontents
% \newpage

\vspace{-1.5em}

\input{section/introduction}

\input{section/main}
\input{section/conclusion}

\clearpage

\bibliographystyle{plainnat}
\bibliography{main}

\clearpage

\input{section/appendix}

% \clearpage

% \newpage

% \beginappendix

% \startcontents[app]
% \begingroup
%   \renewcommand{\contentsname}{Appendix Contents}
%   \section*{\contentsname}
%   \printcontents[app]{}{1}{}
% \endgroup
% \newpage

% \input{section/appendix}

\end{document}

%% file: section/introduction.tex
\section{Introduction}

Diffusion models \citep{diffusion1,score,diffusion3,tu2024motioneditor,tu2024motionfollower,tu2025stableanimator, tu2025stableavatar, tu2025stableanimator++}, trained on web-scale datasets, demonstrate exceptional capability in generating high-fidelity images \citep{image1, sd}. Motivated by this success, researchers have sought to transfer pretrained 2D generative priors to data-scarce modalities. Score Distillation Sampling (SDS) \citep{sds} pioneered this cross-modal knowledge transfer, leveraging pretrained text-to-image diffusion models to optimize 3D differentiable representations through gradient-based maximum likelihood estimation. By circumventing the need for 3D training data, SDS has established text-to-3D generation as a prominent research direction and enabled applications beyond 3D synthesis, including one-step diffusion distillation \citep{dmd, dmd2, dmd_viedo},  character animation \citep{character}, and metric depth prediction \citep{depth}. Despite extensive efforts to refine SDS \citep{vsd, bridge, connect3d, jointdreamer, DreamMesh}, recent studies \citep{dreamreward, dreamalign, dreamdpo} reveal that SDS-synthesized 3D assets often exhibit misalignment with human preferences — a limitation shared by other diffusion models.

% Therefore, to address this issue, reinforcement learning from human feedback (RLHF) is introduced into the pipeline to enhance generative alignment. However, these works usually require additional training of the reward model using 3D data, which undermines the advantage of not requiring 3D data, or inevitably leads to visual artifacts. We argue that it is because the problematic pixel-wise gradient of the reward model disturbs the diffusion process. Among these attempts, dreamdpo [] attempts to construct a DPO-like objective to replace the pixel-wise gradient, but we find it still restricts to the maximum likelihood framework and has limited improvements on alignment. 
% Since the core idea of DPO is to derive an implicit reward model, we draw inspiration from the widely used classifier-free guidance (CFG) [] and question whether \textit{the implicit reward model in score distillation can be considered a special form of guidance that is compatible with PF-ODE.}

To address this misalignment, Reinforcement Learning from Human Feedback (RLHF) has been incorporated into text-to-3D pipelines. However, existing RLHF-based methods \citep{dreamreward, dreamrewardx, dreamcs} typically require training 3D-specific reward models, which fundamentally undermines the core advantage of 3D-data-free synthesis and may induce visual artifacts. Critically, since reward models are exclusively trained on clean images, directly applying their gradients to update 3D representations under high noise levels induces gradient misalignment. Given the established connection between score distillation \citep{sdi,cfd,connect3d,consistent3d} and the \textit{Probability Flow ODE} (PF-ODE) \citep{score}, we hypothesize this originates from pixel-level conflicts between reward gradients and diffusion dynamics. While DreamDPO \citep{dreamdpo} attempts to avoid pixel-wise gradients via DPO-inspired objectives, its disconnection from denoising dynamics limits extensibility to iterative refinement processes.

Motivated by DPO's implicit reward modeling and the efficacy of Classifier-Free Guidance (CFG) \citep{cfg} in conditional diffusion, we propose a fundamental rethinking: \textbf{\textit{Could the implicit reward in score distillation function as a PF-ODE-compatible guidance signal?}}
In this work, we introduce Preference Score Distillation (PSD), a novel framework that harnesses gradients from an implicit reward model to align score distillation with human preferences. To bridge preference learning with guidance mechanisms, we formalize human preference as a  binary variable $\mathcal{S}_{\text{pref}}$, to obtain a preference score guidance term $\nabla_{\boldsymbol{x}_t}\log p\left(\mathcal{S}_{\text{pref}}\mid\boldsymbol{x}_t,y\right)$ that decomposes into interpretable gradient components. Crucially, we identify that suppressing pixel-wise artifacts requires theoretically connecting this preference gradient to the score estimates. This insight motivates a reformulation of RLHF under the score distillation paradigm, where we rewrite the KL-divergence with a dynamic reference distribution tied to the current rendering. Eventually, through our derivation and constructing contrastive sample pair on-the-fly, we formulate a CFG-like guidance that is able to increase the likelihood towards preferred completions. Empirical results proves that it is compatible with existing diffusion dynamics and improve various aesthetics scores directly.

Moreover, noting that the pretrained diffusion model is frozen, we design an algorithm that adaptively updates the preference score and negative text embeddings \citep{neg1, neg2, reneg}. In each denoising step: the preference score is first computed; subsequently, the negative embedding (projected into the continuous text embedding space) is optimized as trainable parameters via reward score backpropagation, and integrated with CFG. Our experiments demonstrate that our approach has strong compatibility and can generate highly photorealistic, preference-aligned 3D assets.

To show the superiority of our method, we compare with state-of-the-art methods that utilizes stronger 3D priors in Fig. \ref{first}. While comparison method \citep{richdreamer, trellis} applies more 3d priors (Normal-Depth diffusion model, physically-based rendering materials) or large-scale training with 3d data, we only distillate diffusion models for image synthesis (MVDream \cite{mvdream} and Stable Diffusion v2.1 \citep{sd}) but still yield better text alignment and aesthetics, which domesticates the novelty of this work.

Our contribution can be summarized as :
\begin{itemize}
    \item We propose a preference alignment method for score distillation, Preference Score Distillation (PSD). To the best of our knowledge, we are the first to demonstrate that \textbf{preference alignment can be directly formulated as a CFG-type guidance} and can produce gradients towards increasing the likelihood of preferred samples via constructing contrastive sample pairs during the optimization process.
    \item We propose a strategy that alternatively updates the preference score and negative embeddings. In this strategy, optimizing continuous negative embeddings can achieve the effect of updating pretrained diffusion parameters.
    \item  Extensive experiments prove the ability of PSD to improve aesthetics scores. The results shows PSD outperforms other comparing methods in scores of 4 human preference assessment models, 1 visual question answering (VQA) model and delivers highly impressive qualitative comparison.
\end{itemize}

%% file: section/main.tex
\section{Preliminaries and Notations}

\textbf{Score Based Diffusion Models.} The process of diffusing a data sample into random noise can be described as \textit{Probability Flow Ordinary Differential Equation} (PF-ODE) \citep{score}. For an arbitrary data point $\boldsymbol{x}_{0} \sim p_{data}$, if we gradually add noise
\begin{equation}
    \boldsymbol{x}_{t} = \alpha_{t}\boldsymbol{x}_{0} + \sigma_{t}\boldsymbol{\epsilon}, \quad \boldsymbol{\epsilon} \sim \mathcal{N}(\mathbf{0}, \mathbf{I}),
\end{equation}
the PF-ODE that has the same marginal distribution can be written as
\begin{equation}
\begin{aligned}
\mathrm{d}(\frac{\boldsymbol{x}_t}{\alpha_t})&=\mathrm{d}(\frac{\sigma_t}{\alpha_t})(-\sigma_t\nabla_{x_t}\log p(\boldsymbol{x}_t))\\
&=\mathrm{d}(\frac{\sigma_t}{\alpha_t})\cdot{\boldsymbol{\epsilon}}_\phi(\boldsymbol{x}_t, t)
\end{aligned}
\label{ode}
\end{equation}
where ${\boldsymbol{\epsilon}}_\phi (\cdot)\approx -\sigma_t\nabla_{\boldsymbol{x}_t}\log p(\boldsymbol{x}_t)$ is our trained diffusion models. Our notations of diffusion models are consistent with \citep{edm, fsd, cfd}.

\textbf{Classifier-free Guidance.} In order to generate text-aligned contents, a technique termed \textit{Classifier-free Guidance} (CFG) \citep{cfg} pushes the samples towards higher likelihood through the gradient of an implicit classifier

\begin{equation}
    \tilde{\boldsymbol{\epsilon}}_{\phi}(\boldsymbol{x}_t, y, t) = \underbrace{\boldsymbol{\epsilon}_\phi(\boldsymbol{x}_t, t)}_{\text{unconditional}} +\underbrace{\gamma(\boldsymbol{\epsilon}_\phi(\boldsymbol{x}_t, y, t) - \boldsymbol{\epsilon}_\phi(\boldsymbol{x}_t, t))}_{\text{implicit classifier } \delta_{cls}}
    \label{cfg}
\end{equation}

where $y$ is the conditioning text embedding and $\gamma$ is a scaling factor, and we denote gradient ${\boldsymbol{\epsilon}}_\phi(\boldsymbol{x}_t, y, t) - {\boldsymbol{\epsilon}}_\phi(\boldsymbol{x}_t, t)$ produced by implicit classifier as $\delta_{cls}$. Additionally, negative prompting has become a common technique to improve generation quality. It replaces $\boldsymbol{\epsilon}_\phi(\boldsymbol{x}_t, t)$ with $\boldsymbol{\epsilon}_\phi(\boldsymbol{x}_t, n,t)$ conditioned by negative embedding $n$. We regard the embedding as a set of the model parameters and thus simplify it as $\boldsymbol{\epsilon}_\phi(\boldsymbol{x}_t, t)$.

\textbf{RLHF on Score Distillation.} Typically, \textit{Reinforcement Learning from Human Feedback} (RLHF) fine-tunes the diffusion models by maximizing expected rewards while regularizing the KL-divergence from a reference distribution \citep{rlhf}. We define a similar objective for score distillation:

\begin{equation}
    \max_{\phi} \mathbb{E}_{\boldsymbol{x}_{t}\sim p_{\phi}(\boldsymbol{x}_{t}|y)} [r(y, \boldsymbol{x}_{t})] - \beta \mathbb{D}_{\text{KL}} [p_{\phi}(\boldsymbol{x}_{t}|y)||q_{\theta}(\boldsymbol{x}_{t}|\boldsymbol{x}_{0}=g_{\theta}(\boldsymbol{c})) ],\label{rlhf_ours}
\end{equation}
where $\theta$ is the learnable parameters of differentiable representation $g$, $\boldsymbol{c}$ is the rendering camera view,  $q_{\theta}(\cdot)$ is the marginal distribution and $\mathbb{D}_{\text{KL}}[\cdot]$ is the KL-divergence. The difference between our definition and standard RLHF \citep{rlhf} is we modify the reference model in KL term into marginal distribution of current rendering $g_{\theta}(\boldsymbol{c})$ and we seek to optimize for an arbitrary timestep $t$. Justifications are presented in Appendix \ref{justification}, where existing works \citep{dreamreward, dreamdpo} can be related to our definition.

\begin{figure}[t] 
\centering 
\includegraphics[width=\textwidth]{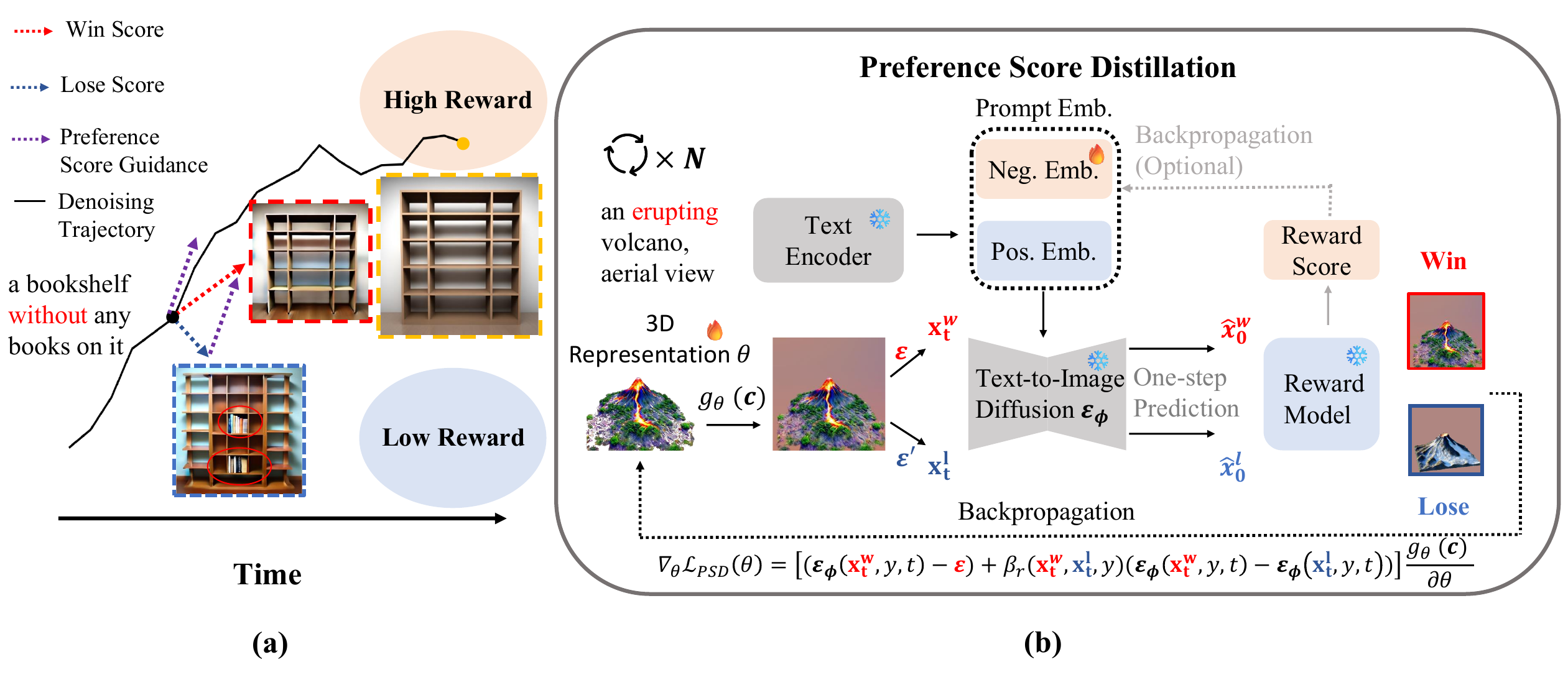} 
\caption{Overall illustration of Preference Score Distillation. a) Win (red) and lose (blue) samples are constructed on-the-fly to calculate win and lose scores, then preference score guidance (purple) pushes the denoising trajectory towards high-reward regions and finally improve alignment with reward.  b) In each step, two noise is added to the rendering images $g_{\theta}(\boldsymbol{c})$ and reward model determines win/lose based on one-step prediction of pretrained diffusion models $\boldsymbol{\epsilon}_\phi$. 3D representation $\theta$ and negative embedding $n$ are updated by our objective $\mathcal{L}_{\text{PSD}}(\theta)$ and reward score respectively.} 
\label{overall} 
\end{figure}

\section{Approach}
\label{approach}

In Section \ref{link}, we first establish the connection between preference and guidance by deriving the preference score guidance. In Section \ref{psd}, we present the proposed preference score distillation (PSD) method. In Section \ref{update}, we introduce a novel adaptive strategy for updating the preference score and negative embedding to improve the quality of generation. Due to the limited space, we present overall pseudo code of PSD in Algorithm \ref{alg_psd}.

\subsection{Linking Preference to Guidance}
\label{link}

The foundation of our framework is built upon a recent insight \citep{connect3d,cfd,consistent3d} into the connection between differential representation optimization and the denoising process of diffusion models. In Eq. \ref{insight}, $\mathrm{d}(\frac{\sigma_t}{\alpha_t})$ can be viewed as the learning rate $lr$ of an optimizer and $\boldsymbol{\epsilon}_\phi(\cdot)$ can be viewed as the gradient $\nabla \mathcal{L}$ of $\mathrm{d}(\frac{\boldsymbol{x}_t}{\alpha_t})$. In order to guiding $\boldsymbol{x}_t$ with preference, we formally introduce a binary variable $\mathcal{S}_{\text{pref}}$ as human-preferred properties for constrained conditions. 

\begin{equation}
\mathrm{d}(\frac{\boldsymbol{x}_t}{\alpha_t})=\underbrace{\mathrm{d}(\frac{\sigma_t}{\alpha_t})}_{-lr}\cdot\underbrace{\boldsymbol{\epsilon}_\phi(\boldsymbol{x}_t,y,\mathcal{S}_{\text{pref}},t)}_{\nabla \mathcal{L}}.
\label{insight}
\end{equation}

and apply Bayes’ rule 

\begin{equation}
  \nabla_{\boldsymbol{x}_t}\log p_\phi\left(\boldsymbol{x}_t\mid y,\mathcal{S}_{\text{pref}}\right)=\nabla_{\boldsymbol{x}_t}\log p_\phi\left(\boldsymbol{x}_t\mid y\right)+\nabla_{\boldsymbol{x}_t}\log p_{\phi}\left(\mathcal{S}_{\text{pref}}\mid\boldsymbol{x}_t,y\right).  
\label{guidance}
\end{equation}

A naive solution is to train a "classifier" (reward model) to estimate the probability of $\boldsymbol{x}_t$ aligning with human preference $p_{\phi}\left(\mathcal{S}_{\text{pref}}\mid\boldsymbol{x}_t,\boldsymbol{c}\right)$, but the drawback majorly involves: 1) introducing additional training to produce appropriate gradients and 2) reward models needs to perform at arbitrary timestep, but it can be only trained with clean images so we have to approximate clean data during early steps when noise is large \citep{image1}. Thus, we seek to formulate CFG-type guidance.

To achieve this, our high-level idea is to construct gradient via win-lose pair similar to DPO \citep{dpo,dspo} such that the denoising process is pushed to increase the likelihood of wining sample. We first introduce Bradley-Terry (BT) model \citep{bt} for human preference

\begin{equation}
\begin{aligned}
    p\left(\mathcal{S}_{\text{pref}} \mid \boldsymbol{x}_t, y \right)
& = p\left(\boldsymbol{x}_t^{\text{w}}\succ \boldsymbol{x}_t^{\text{l}} \mid \boldsymbol{x}_t, y \right)
 = \sigma\left( \Delta r_t \right),
\end{aligned}
\label{bt}
\end{equation}
where $\sigma(\cdot)$ represents the sigmoid function, $\Delta r_t=r\left(y, \boldsymbol{x}_t^{\text{w}} \right) - r\left(y, \boldsymbol{x}_t^{\text{l}} \right)$, $r(\cdot)$ is reward model, and $\boldsymbol{x}_t^{\text{w}}\succ \boldsymbol{x}_t^{\text{l}}$ represents  $\boldsymbol{x}_t^{\text{w}}$ and $\boldsymbol{x}_t^{\text{l}}$ are win-lose pair examples respectively. Different from off-line preference optimization, {we construct $\boldsymbol{x}_t^{\text{w}}$ and $\boldsymbol{x}_t^{\text{l}}$ in Eq. \ref{bt} online in order to link preference with inference-time guidance}. Consequently, replace $p_{\phi}\left(\mathcal{S}_{\text{pref}}\mid\boldsymbol{x}_t,y\right)$ in Eq. \ref{guidance} with Eq. \ref{bt} then Eq. \ref{insight} becomes

\begin{equation}
\mathrm{d}\left(\frac{\boldsymbol{x}_t}{\alpha_t}\right) = \mathrm{d}\left(\frac{\sigma_t}{\alpha_t}\right) \cdot \Big( \boldsymbol{\epsilon}_\phi(\boldsymbol{x}_t, y, t) - \sigma_t \left( 1 - \sigma (\Delta r_t) \right) \nabla_{\boldsymbol{x}_t}( \Delta r_t) \Big).
\label{preference}
\end{equation}

The primary challenge of solving Eq. \ref{preference} is term $\nabla_{\boldsymbol{x}_t}(\Delta r_t)$ being not tractable. Fortunately, we can rewrite the reward using the unique global optimal solution $p^{*}_{\phi}$ of Eq. \ref{rlhf_ours} varied from \citep{dpo}, which is

\begin{equation}
    r(y,\boldsymbol{x}_{t}) = \beta \log \frac{p^*_{\phi}(\boldsymbol{x}_{t}|y)}{q_{\theta}(\boldsymbol{x}_{t}|\boldsymbol{x}_{c}=g_{\theta}(\boldsymbol{c}))} + \beta \log Z(\boldsymbol{x}_{c}),
\label{reward}
\end{equation}
where $Z(\boldsymbol{x}_{0})$ is a trivial partition function since it eliminates directly. Then, by plugging Eq. \ref{reward} into Eq. \ref{preference}, we can propose our preference guided ODE.

\begin{equation}
    \mathrm{d}\left(\frac{\boldsymbol{x}_t}{\alpha_t}\right) = {} \mathrm{d}\left(\frac{\sigma_t}{\alpha_t}\right) \cdot \Biggl( \underbrace{\boldsymbol{\epsilon}_\phi(\boldsymbol{x}_t,y,t)}_{(A)} + \beta\sigma\left(- \Delta r_t \right) \underbrace{\left(\boldsymbol{\epsilon}_\phi(\boldsymbol{x}_t^w,y,t)-\boldsymbol{\epsilon}_\phi(\boldsymbol{x}_t^l,y,t)\right)}_{(B)} \Biggr).
\label{expansion}
\end{equation}

We neglect some of the less important terms for simplicity. Full details are in Appendix \ref{derivation}. Observe Eq. \ref{expansion}, {if we force a frozen pretrained diffusion model (substituting $p^*_{\phi}$ with $p_{\phi}$)}, term (A) will be analogous to unconditional score in Eq. \ref{cfg}, and newly introduced term (B) will be analogous to the implicit classifier $\delta_{cls}$. As for term $\sigma\left( - \Delta r_t \right)$, it weights the guidance term by how  incorrectly the implicit reward model ranking the win-lose pair. Generally, we name the term (B) as \textit{preference score guidance} and denote as $\delta_{pref}$. 

In practice, we calculate the reward on noisy steps $r\left(y, \boldsymbol{x}_t \right)$ with the approximation $r\left(y, \hat{\boldsymbol{x}}_0 \right)$ based on Tweedie's formula \citep{tweedie} $\hat{\boldsymbol{x}}_{0} = \frac{\boldsymbol{x}_t - \sigma_t \boldsymbol{\epsilon}_\phi(\boldsymbol{x}_t, y, t)}{\alpha_{t}}$, and choose the sample with higher score to be $\boldsymbol{x}_t^w$. Furthermore, to reduce the number of forward passes in each denoising step, we replace term (A) with $\boldsymbol{\epsilon}_\phi(\boldsymbol{x}_t^w,y,t)$ and apply CFG to calculate the win/lose score in $\delta_{pref}$ as well.

\subsection{Preference Score Distillation}
\label{psd}

A key feature while formulating Eq. \ref{expansion} is the introduction of on-the-fly win-lose pair. 
% Classic solvers for diffusion models such as DDIM [] and DPM-Solver [] are not directly applicable even in 2D scenarios. Fortunately, 
For score distillation, \citep{dreamdpo} has shown the use of different noise to construct win-lose pair. In our Eq. \ref{expansion}, if we define

\begin{equation}
	\boldsymbol{x}_{t}^w = \alpha_{t}\boldsymbol{x}_{c} + \sigma_{t}\boldsymbol{\epsilon}, \quad \boldsymbol{x}_{t}^l = \alpha_{t}\boldsymbol{x}_{c} + \sigma_{t}\boldsymbol{\epsilon}',
	\label{win-lose}
\end{equation}
where $\boldsymbol{x}_{c}$ is a non-noisy sample (whether it is one-step predicted samples from \citep{sdi} or clean space samples from \citep{csd,fsd}), $\boldsymbol{\epsilon}, \boldsymbol{\epsilon}' \sim \mathcal{N}(\mathbf{0}, \mathbf{I})$ are two independent noise, then our preference guidance $\delta_{pref}$ will become a velocity field that pushes the samples towards high reward regions in all conditions. Eventually, with the \textit{change-of-variable} \citep{sdi, fsd, cfd} technique already discussed in previous works, guiding 3D generation using Eq. \ref{expansion} can be formulated as objective

\begin{equation}
	\nabla_{\theta} \mathcal{L}_{\text{PSD}}(\theta) = \mathbb{E}_{t,\boldsymbol{c}} \left[ \left( \delta_{gen} + \gamma \delta_{cls} + \beta_{r}\delta_{pref} \right) \frac{\partial g_{\theta}(\boldsymbol{c})}{\partial \theta} \right],
    \label{objective}
\end{equation}

where unconditional prior $\delta_{gen}={\boldsymbol{\epsilon}}_\phi(\boldsymbol{x}_t^w, t) - \boldsymbol{\epsilon}$ and we set $\beta_{r}=\gamma\frac{||\delta_{cls}||_2}{||\delta_{pref}||_2}\cdot\sigma\left( -\Delta r_t \right)$ to balance the gradient produced by CFG and our preference score guidance. Notice that $\delta_{gen}$ is a general formulation for score distillation \citep{sds, cfd, bridge, csd, segment} methods although their derivations may differ. We present the illustration of overall mechanism in Fig. \ref{overall}.

\subsection{Adaptive Update of Preference Score and Negative Embeddings} 
\label{update}

\begin{wrapfigure}{r}{9cm}
\centering
\includegraphics[width=\linewidth]{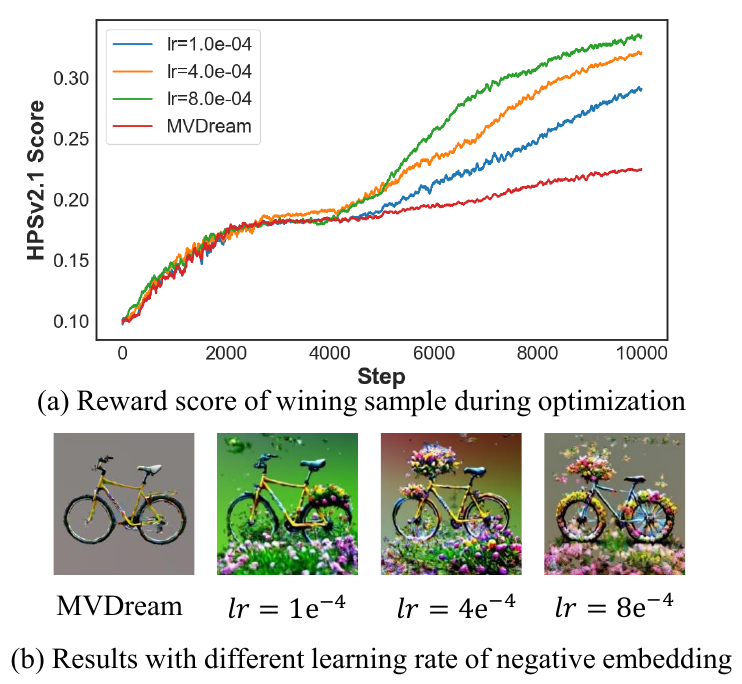}
\caption{Effect of negative embedding optimization strategy on single prompt. Employ negative embedding optimization can significantly improve aesthetic score, but overlarge learning rate will harm visual quality. }
\label{lr}
\end{wrapfigure}

In Section \ref{link}, we assumed a frozen diffusion model $p_{\phi}$. However, it can impose limitations on the effectiveness of our proposed preference guidance, especially when 3D generation via score distillation usually requires in thousands of denoising steps. The reason behind is  $p_{\phi}$ is never updated so that the optimal solution $p_{\phi}^*$ can never be approached. A high-level theoretical analysis is presented in Appendix \ref{appendix_proof}. On the other hand, since we only perform prompt-specific optimization, it is not rational to update the pretrained model itself even only part of the parameters (e.g.  LoRA \citep{lora} implementations in VSD \citep{vsd}).

Therefore, inspired by ReNeg \citep{reneg}, which regards the negative embedding as part of the model parameters, we develop an algorithm which adaptively update the preference score and negative embeddings.  Specifically, we initialize the negative embedding with hand-crafted negative descriptors and negative embedding $n$ is updated by maximizing

\begin{equation}
	\mathcal{L}_{\text{Emb}}(n) = \mathbb{E}_{\boldsymbol{c}} \left[ r(y, \hat{\boldsymbol{x}}_0) \right],
    \label{embedding}
\end{equation}

To enable training with negative embedding, $\hat{\boldsymbol{x}}_0$ is the same one-step prediction used in Eq. \ref{expansion}, so that negative embedding can be involved while incorporating with CFG. We find it enough to share the same negative embedding for all viewing direction $\boldsymbol{c}$.

For better viewing the effect of our proposed negative embedding optimization technique, Fig. \ref{lr} illustrates the curves of target reward score during optimization with different learning rates and their respective results. Comparing the end-point score and the visual quality, larger learning rate will bring benefits of higher score but result in "reward hacking" \citep{test, hacking}. Thus, we make several practical trade-off (presented in Appendix \ref{regularizer}). Eventually, we achieve aesthetic score improvements that align with human perception.

\section{Experiments}
\label{experiment}

\subsection{Experimental Setups}

\begin{figure}[t] 
\centering 
\includegraphics[width=\textwidth]{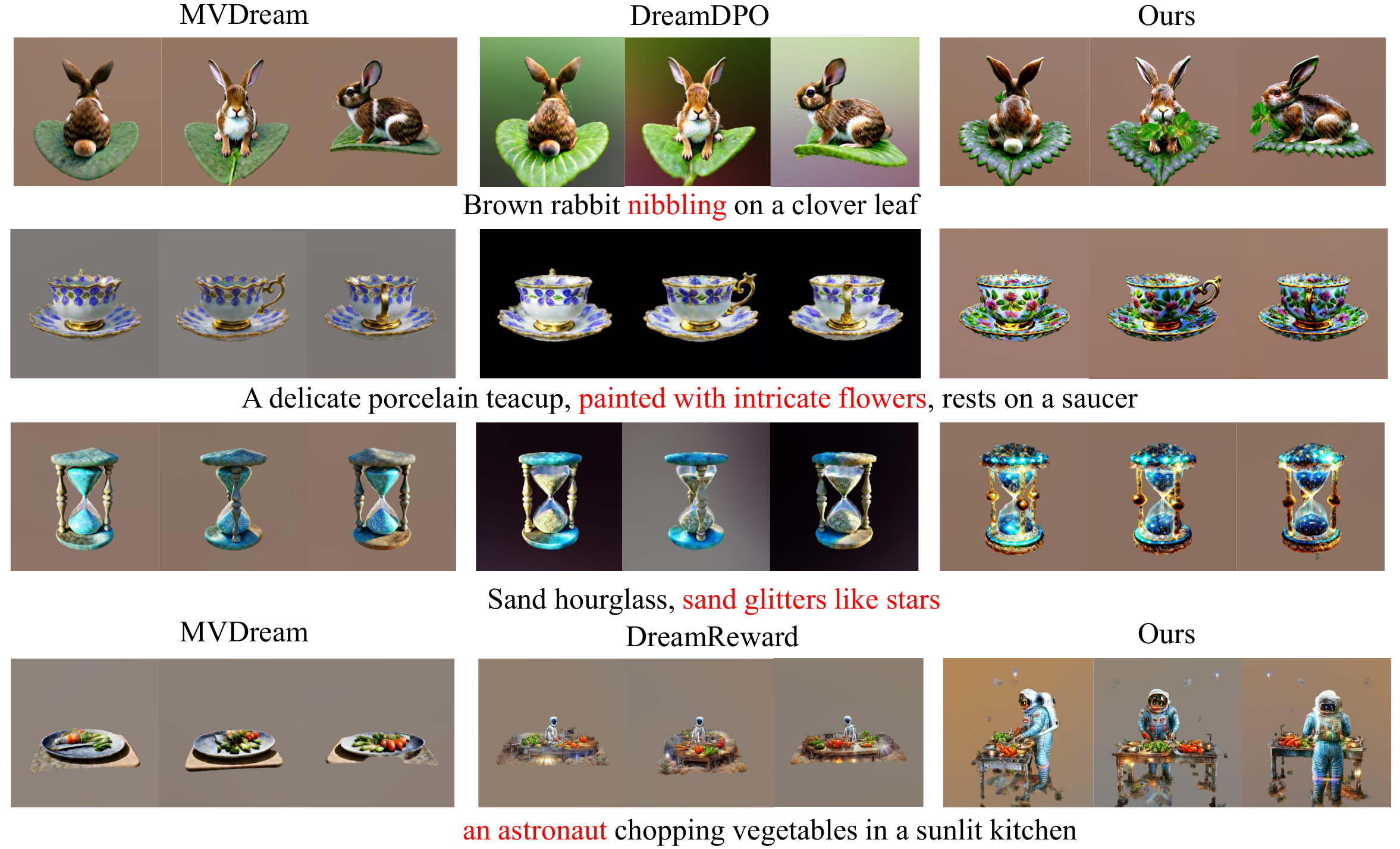} 
\caption{Qualitative comparison of single-stage distillation of MVDream \citep{mvdream} ($256\times256$). Our PSD significantly improve text alignment (red) and visual quality against comparing method DreamReward \citep{dreamreward} and DreamDPO \citep{dreamdpo}.} 
\label{mvdream_fig} 
\end{figure}

\begin{wraptable}{r}{10cm}

{
 \caption{Quantitative Comparison of single-stage distillation of MVDream \citep{mvdream} across 200 prompts in Eval3d \citep{eval3d} with different rewards. Higher values are better ($\uparrow$). The best performance is in bold. }
 \label{mvdream}
}
 \setlength{\tabcolsep}{0.5mm}{
\begin{tabular}{c|lccccc}
\toprule
\textbf{Rewards} &\textbf{Methods} & \textbf{I.R. $\uparrow$} &  \textbf{Pick. $\uparrow$} & \textbf{Aes. $\uparrow$}  & \textbf{MPS} $\uparrow$ & \textbf{T.A.} $\uparrow$\\
\midrule
\multirow{3}*{HPSV2.1}& MVDream & -0.22  & 20.55 & 5.79 & 9.30 & 53.14\\
% CFD & -0.43 & 0.21 & 20.34  & & & 73.63\\
% JointDreamer & \\
&DreamDPO & -0.28 & 20.48 & 5.80 & 9.00 & 75.68\\
&Ours &\textbf{ 0.12}  & \textbf{20.99} & \textbf{5.92} & \textbf{10.26} & \textbf{75.70} \\

\midrule
\multirow{2}*{Reward3D}&DreamReward & 1.78  & 21.40 & 6.15 & 10.19 & 74.37 \\
&Ours & \textbf{1.80}  & \textbf{21.49} & \textbf{6.25} & \textbf{10.40} & \textbf{90.63} \\
% \, + CFD \\
\bottomrule
\end{tabular}
}
\end{wraptable}

To assess the performance, our experiments include various settings under the same codebase \textit{threestudio} \citep{threestudio}. For direct comparisons with existing score distillation preference alignment methods, we use 200 test prompts from Eval3d \citep{eval3d} to perform a one-stage distilling of MVDream \citep{mvdream}. To justify our proposed PSD on high-resolution generation, we further evaluate on several more complex pipelines. In the 2-stage NeRF \citep{nerf} synthesis and 3-stage DMTet \citep{dmtet} synthesis, we first distill MVDream and then Stable Diffusion v2.1 \citep{sd}. Since these pipelines require more optimization time, we evaluate on a more difficult filtered 40-prompt subset presented in Appendix \ref{prompt}. For target reward, we apply HPSv2.1 \citep{hps} if without specification. As for evaluation metrics, we assess with human preference reward ImageReward \citep{imagereward} (I.R.), PickScore \citep{pickscore} (Pick.), Aesthetic scores \citep{laion} (Aes.), Multi-dimensional Preference Score \citep{mps} (MPS) and VQA model Qwen2.5-VL-7B \citep{qwen} for text-3D alignment (T.A.) using the question-answer pair generated in Eval3d. More implementation details are presented in Appendix \ref{details}.

\subsection{Results}
\label{results}

\textbf{Quantitative comparisons.} Shown in Tab. \ref{mvdream} and \ref{complex}, our PSD outperforms all other methods, which indicates improvements on generation quality as well as alignment with human preference. Specifically, comparing with DreamDPO \citep{dreamdpo}, a method that uses pretrained 2D reward, we achieve a more significant improvements, highlighting our ability to unleash 2D reward. As for cooperating with method that requires additional training represented by DreamReward \citep{dreamreward}, our results still showcases our advantage on text-3D alignment. Besides, Reward3D finetuned in DreamReward enforces 4 input view and is not memory feasible for high-resolution generation. Additionally, we compare with RichDreamer \citep{richdreamer}, a method that introduce extra multi-view normal, depth, albedo diffusion priors. The results proves that misalignment of human preference commonly exists  even  when stronger diffusion priors are employed.

\begin{figure}[t] 
\centering 
\includegraphics[width=\textwidth]{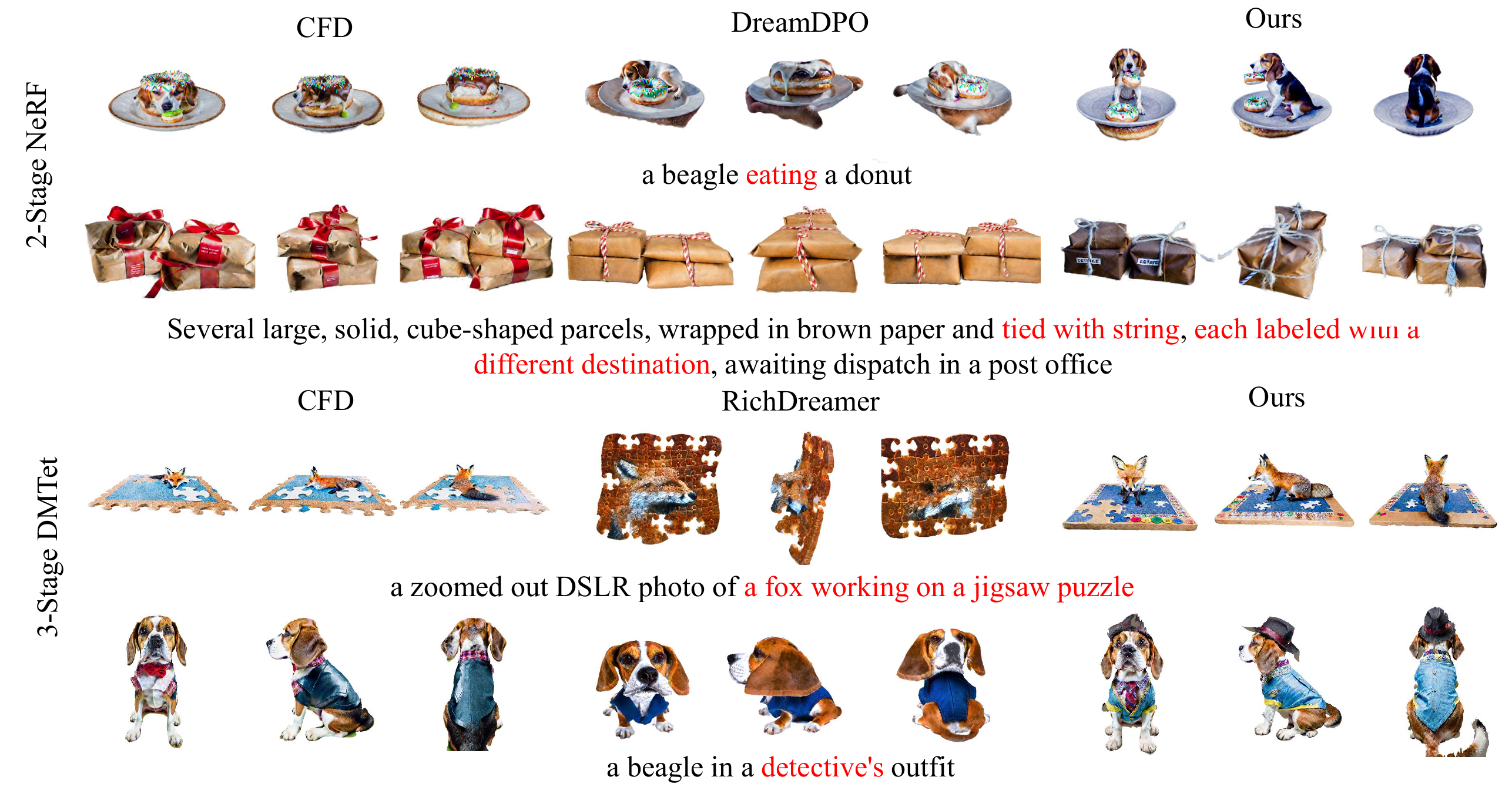} 
\caption{Qualitative comparison of 2-stage NeRF \citep{nerf} ($512\times512$) and 3-Stage DMTet \citep{dmtet} ($1024\times1024$) generation. PSD improves alignment with the prompts in red.} 
\label{2stage_fig} 
\end{figure}

\begin{table}[!t]

\centering
{
 \caption{Quantitative Comparison of 2-stage NeRF \citep{nerf} and 3-Stage DMTet \citep{dmtet} generation across 40-prompt subset. Higher values are better ($\uparrow$). The best performance is in bold. Comparison method DreamReward \citep{dreamreward} suffer from out-of-memory (OOM). * indicates method uses stronger diffusion priors. }
 \label{complex}
}
 \setlength{\tabcolsep}{1.8mm}{
\begin{tabular}{c|lccccc}

\toprule
\textbf{Pipeline} &\textbf{Methods} & \textbf{I.R. $\uparrow$} &  \textbf{Pick. $\uparrow$} & \textbf{Aes. $\uparrow$}  & \textbf{MPS} $\uparrow$ & \textbf{T.A.} $\uparrow$\\
\midrule
\multirow{4}*{2-Stage NeRF} & CFD  & -0.09  & 20.16  & 5.63  & 10.03 & 71.84 \\
&DreamReward & \multicolumn{5}{c}{OOM}\\
&DreamDPO & -0.06 & 20.18 & 5.40 & 10.04 & 76.18\\
&Ours & \textbf{0.01} & \textbf{20.27} &\textbf{ 5.95} & \textbf{10.22} & \textbf{81.34}\\
\midrule
\multirow{3}*{3-Stage DMTet}& CFD & -0.44 & 19.76 & 5.30  & 9.45 & 71.84\\
 & RichDreamer*& \textbf{0.02} & 19.68 &\textbf{ 5.92} & 8.24 & 76.45 \\
 & Ours  & -0.40 & \textbf{19.92} & 5.34 &\textbf{ 9.64} & \textbf{81.81} \\

\bottomrule
\end{tabular}
}
\end{table}

\textbf{Qualitative comparisons.} We provide qualitative comparison in Fig. \ref{mvdream_fig} and \ref{2stage_fig}. Baselines and previous methods deviate from given prompt text, while employing our PSD can lead to macroscopic improvements on text alignment (marked in red) and visual details. Also, noticing DreamReward will introduce artifacts, we provide supplementary evaluations of geometry quality in Appendix \ref{geometry}. Noticing the misalignment of quantitative and qualitative comparisons, we believe this is also due to reward hacking. Limited by the space, more discussion is presented in Appendix \ref{comparision}.

\begin{wrapfigure}{r}{9cm}
\centering
\includegraphics[width=\linewidth]{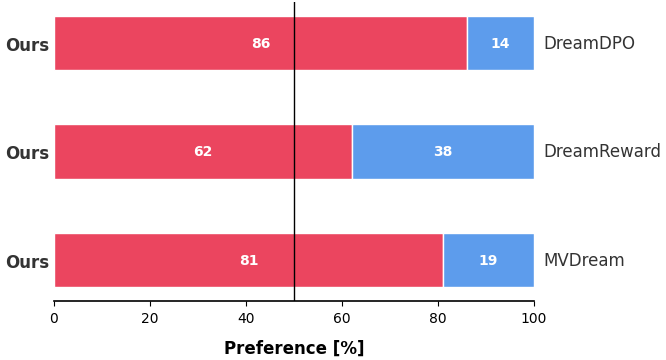}
\caption{{User preference study of comparing PSD with MVDream, DreamDPO, and DreamReward.} }
\label{user}
\end{wrapfigure}

\textbf{User study.} To validate the efficacy of our proposed method to real human users, we present a user study involving 24 participants. They are required to choose the better one for each comparison across four dimensions: Appearance Quality, Structure Quality, Text Alignment, and Overall Performance. The survey consists of 30 pairs of videos created from MVDream, DreamDPO, DreamReward, and our PSD. Shown in Fig. \ref{user}, our PSD received higher preference score, which is consistent with the results previously given by reward models. More details of this user study is included in Appendix \ref{app_user}.

\subsection{Ablation Study}
\label{ablation}

\begin{figure}[t] 
\centering 
\includegraphics[width=0.9\linewidth]{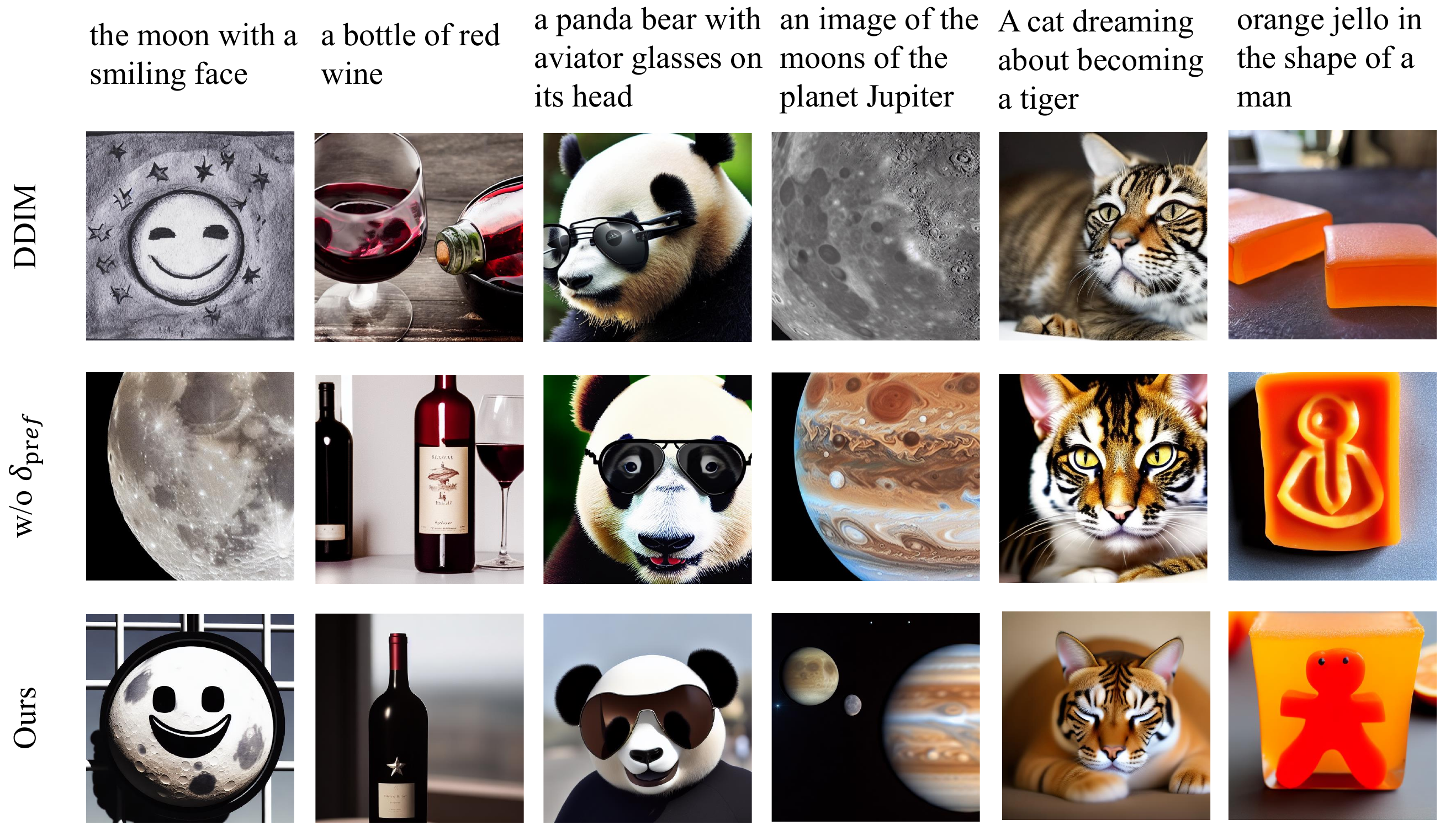} 
\caption{Preference score guidance on image generation. } 
\label{2dvisiual} 
\end{figure}

\begin{table}[!t]

    \centering
    {
\caption{Ablations of preference score on 2D image generation.  Since our final results and experiments without $\delta_{pref}$ are performed using optimization, standard DDIM \citep{diffusion3} are marked in gray. \label{tab_2d} }
}
    \begin{tabular}{lcccc}
\toprule
 \multirow{2}*{\textbf{Methods}} & {\textbf{Target}} & \multicolumn{3}{c}{\textbf{Unseen Rewards}}  \\
\cmidrule(lr){2-2} \cmidrule(lr){3-5} 
 & \textbf{HPSv2.1 $\uparrow$} & \textbf{I.R. $\uparrow$} & \textbf{Pick. $\uparrow$} & \textbf{MPS $\uparrow$}   \\
\midrule
\color{gray}DDIM & \color{gray}0.25 & \color{gray}\textbf{0.22} & \color{gray}\textbf{21.30} & \color{gray}9.77\\
\midrule
w/o $\delta_{pref}$ & 0.25 & 0.04 & 21.15 & 9.51\\
Ours (Eq. \ref{objective}) & \textbf{0.26} & \textbf{0.22} & 21.23 & \textbf{10.30} \\
\bottomrule
\end{tabular}
\end{table}

\textbf{Ablations on negative embedding optimization.} In Fig. \ref{lr}, we analysis the behavior of our negative embedding optimization strategy on single prompt. Quantitative ablation is presented in Tab. \ref{ablation_neg}. We evaluate on the 40-prompt subset with MVDream, PSD with preference score guidance only (w/o $n^*$) and PSD with negative embedding learning rate $4e^{-4}$ ($lr=4e^{-4}$). The results verify that our practical trade-off is necessary such that appropriate negative optimization can benefit improving aesthetic scores.

\textbf{Ablations on different reward models.} Theoretically, PSD is compatible with any pretrained 2D reward model. To present the difference, we incorporate ImageReward \citep{imagereward}. Shown in FIg. \ref{ablation_reward}, different reward models may benefit performance (capturing concept like "streaming" and "growing on a log"), but for the sake of fair comparison, we don't introduce any new models and inherit from the comparison methods instead.

\begin{wrapfigure}{r}{6cm}
\centering
\includegraphics[width=\linewidth]{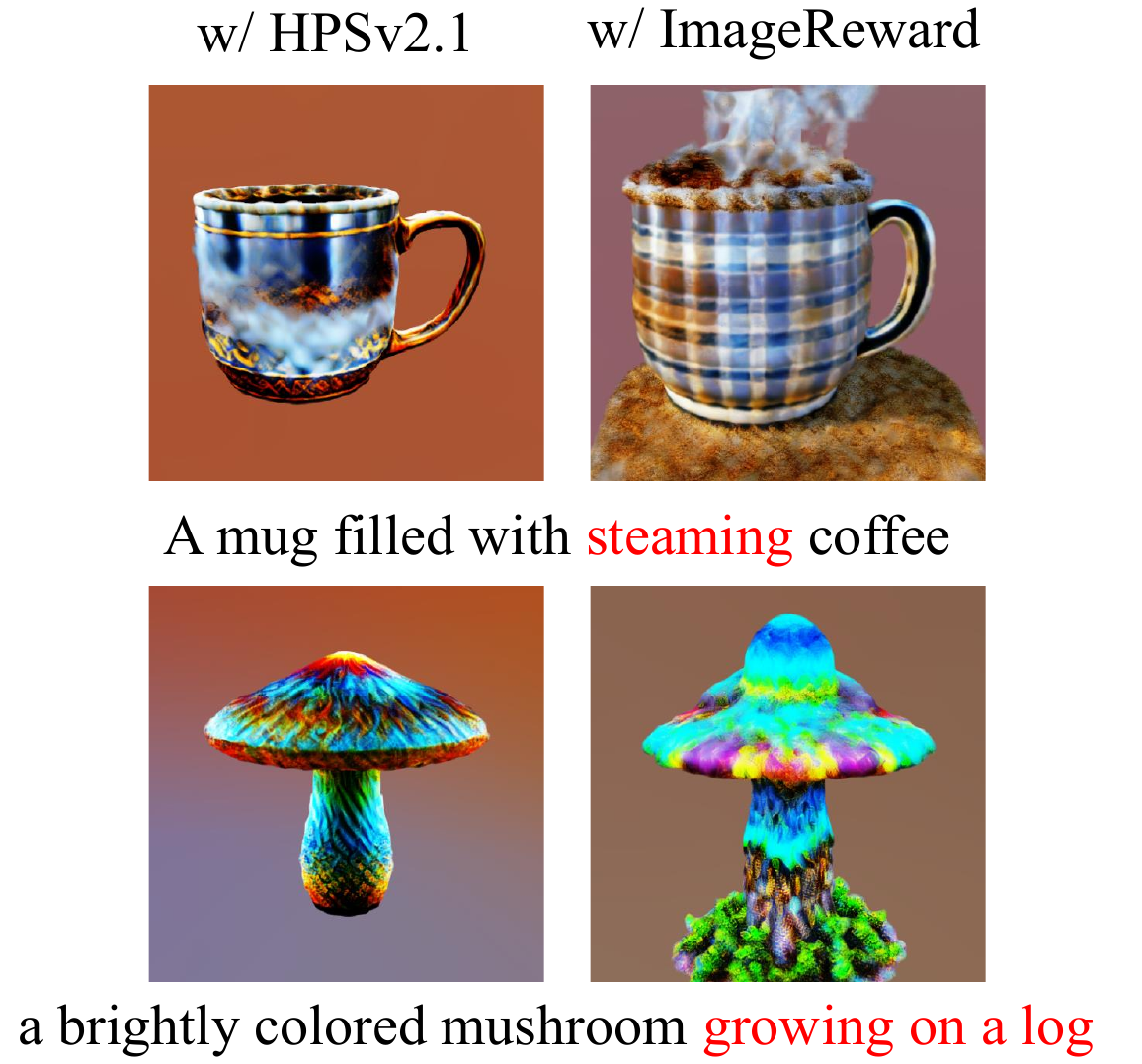}
\caption{Ablation of choosing different reward models. }
\label{ablation_reward}
\end{wrapfigure}

\textbf{Ablations on the preference score.} To justify the compatibility of our preference guidance with PF-ODE, we directly use Eq. \ref{objective} to guide the diffusion process of parameterized images. Following common configurations of DDIM \citep{diffusion3}, we update the latents 50 steps but with a Adam \citep{adam} optimizer. This is based on the fact that image can also be parametrized \citep{dreamsampler, centric}.  Specially, to share the same dynamics with DDIM, we replace one of 
$\boldsymbol{\epsilon}, \, \boldsymbol{\epsilon}'$ in Eq. \ref{win-lose} with the noise $\boldsymbol{\epsilon}_\phi(\boldsymbol{x}_s,y,s)$ predicted at more noisy timestep $s$ as discussed in \citep{sdi}.
Evaluations with Stable Diffusion v1.5 on Parti-Prompt dataset \citep{parti} is presented in Tab. \ref{tab_2d} and direct visual comparison is shown in Fig. \ref{2dvisiual} . Comparing with results of DDIM and optimizing Eq. \ref{objective} without $\delta_{pref}$, introducing the preference guidance will significantly improve reward scores and even outperform the original DDIM. It is a strong evidence to prove that our preference guidance can lead the generation process towards high-reward regions. Besides, comparing the scores of target and unseen rewards, using preference guidance alone will reduce the risk of reward hacking. Pseudo code of this experiment is listed in Appendix \ref{2d_details}.

\begin{table}[t]

    \centering
    {
\caption{Ablations on negative embedding optimization strategy. w/o n represents experiments without negative embedding optimization. $lr=4e^{-4}$ represents experiments with higher learning rate of learnable embedding. $1/5it^{-1}$ and $1/2it^{-1}$ represent performing preference guidance and negative embedding optimization with a interval of 2 and 5 steps respectively.\label{ablation_neg}}
}
    \begin{tabular}{lccccc}
\toprule
\multirow{2}*{\textbf{Experiments}} & {\textbf{Target}} & \multicolumn{3}{c}{\textbf{Unseen Rewards}} & \multirow{2}*{\textbf{Generation Speed}}  \\
\cmidrule(lr){2-2} \cmidrule(lr){3-5} 
 & \textbf{HPSv2.1 $\uparrow$} & \textbf{I.R. $\uparrow$} & \textbf{Pick. $\uparrow$} & \textbf{MPS $\uparrow$}   \\
\midrule
MVDream & 0.20 & -0.63 & 19.22 & 8.39 & \textbf{2.92 it/s}\\
DreamDPO &  0.19& -0.70 & 19.15 & 8.58 & 1.45 it/s\\
w/o $n$ & 0.20 & -0.60 & 19.28 & 8.51 & 1.45 it/s\\
$lr=4e^{-4}$& \textbf{0.24} & -0.49 & 19.85 & 9.98 & 1.19 it/s\\
\midrule
$1/5it^{-1}$ & 0.21 & -0.38  & 19.45 & 9.35 & 2.45 it/s\\
$1/2it^{-1}$ & 0.22 & -0.25 & 19.72 & 9.70 & 1.50 it/s\\
Ours & 0.23 & \textbf{-0.25} & \textbf{19.99} & \textbf{10.12} & 1.19 it/s\\
\bottomrule
\end{tabular}
\end{table}

\textbf{Time Consumption.} Since the preference guidance is formulated into a CFG-style term, it can be calculated within one forward pass of diffusion model. As a consequence, if the reward model also supports batch inference, the only additional computational consumption for each step is to update the negative embedding. Considering that the size of negative embedding is comparatively small, times consumption should be acceptable. Results are in Tab. \ref{ablation_neg}, where we provide two additional settings. A larger interval between adjacent reward signal leads to significant accelerations while our method still outperforms baselines, showing the efficacy of our proposal.

%% file: section/conclusion.tex
\section{Conclusion}
In this paper, we propose Preference Score Distillation, which basically implies that preference optimization in score distillation can be regarded as a type of guidance. To link preference with guidance, we start by deriving preference score guidance under our modified definition of RLHF. By constructing win-lose pair on-the-fly, we achieve effective guidance to the optimization process of score distillation. We also develop an adaptive update strategy to unleash the potential of preference score guidance, noticing the parameters of pretrained diffusion models are not updated. During the entire procedure, we successfully avoid additional training of reward models using 3d data, and demonstrate significant improvements on generating highly photorealistic, human preference aligned 3d objects.
\newpage

%% file: section/appendix.tex
\appendix
\section*{Appendix}
\section{Table of Notations}

We use consistent notations across the main paper and supplementary materials, which are listed in Tab. \ref{notations}

\begin{table}[h]
\centering
\caption{List of notations and their descriptions.}
\label{notations}
\begin{tabular}{ll}
\toprule
\textbf{Notation} & \textbf{Description} \\
\midrule
$\boldsymbol{x}_{t}$ & State variable at timestep $t$ \\
$\alpha_t,\sigma_t$ & Time-dependent coefficients \\
$\epsilon, \epsilon'$ & Noise sampled from Gaussian distribution \\
$\boldsymbol{\epsilon}_\phi(\cdot)$ & Diffusion models parameterized by $\phi$ \\
$y$ & Conditioning text prompt embedding \\
$n$ & Negative prompt embedding \\
$\gamma$ & Scaling factor in CFG \\
$\beta$ & Scaling factor in RLHF \\
$r,\, \Delta r_t$ & Reward, $\Delta r_t=r\left(y, \boldsymbol{x}_t^{\text{w}} \right) - r\left(y, \boldsymbol{x}_t^{\text{l}} \right)$ \\
$g_{\theta}(\cdot)$ & Differential 3D representation parameterized by $\theta$  \\
$q_{\theta}(\cdot)$ & Marginal distribution determined by parameter $\theta$  \\
$\boldsymbol{c}$ & Rendering camera view  \\
$\mathbb{D}_{\text{KL}}[\cdot]$ & KL-divergence \\
$\mathcal{S}_{\text{pref}}$ & Binary variable represents human preference \\
$\boldsymbol{x}_c$ & Non-noisy sample \\
$\hat{\boldsymbol{x}}_0$ & One-step prediction sample based on Tweedie's formula $\hat{\boldsymbol{x}}_{0} = \frac{\boldsymbol{x}_t - \sigma_t \boldsymbol{\epsilon}_\phi(\boldsymbol{x}_t, y, t)}{\alpha_{t}}$\\
$\boldsymbol{x}_t^w,\boldsymbol{x}_t^l$ & Win-lose sample pair at timestep $t$\\
$\beta_r$ & Adaptive scaling factor $\beta_{r}=\gamma\frac{||\delta_{cls}||_2}{||\delta_{pref}||_2}\cdot\sigma\left(-\Delta r_t\right)$ \\
\bottomrule
\end{tabular}
\end{table}

\section{Related Works}

\subsection{Diffusion Models}
 Diffusion model \citep{diffusion1,score,diffusion3} is a class of generative models that learn to reverse a diffusion process which gradually adding noise to a data distribution. With the introduction of latent space \citep{sd}, DM has proven its scaling ability to generate high-dimensional, perceptual data such as image \cite{sd3, sdt} and videos \cite{svd}. Among the dense theory behind, interrelating the diffusion process into a \textit{Probability Flow Ordinary Differential Equation} (PF-ODE) or \textit{Stochastic Differential Equation} (SDE) \citep{score} is an essential step towards a unified framework in pursuit of mathematically guaranteed high-quality generation. In representation, \textit{Stable Diffusion} \citep{sd} has developed a stack of variations that significantly promote the quality and efficiency of photorealistic generation. 

\subsection{Preference Alignment of Diffusion Models}
Although web-scale pretraining enables promising performance to diffusion models, they may deviate users' preference. To overcome the issue, a line of works focuses on using \textit{Reinforcement Learning from Human Feedback} (RLHF) for fine-tuning. For example,  \cite{hps} trained a human preference classifier to employ supervised fine-tuning with preference-based reward model. \cite{imagereward} also trained reward model but they directly maximize reward through backpropagation of differentiable scores. DPOK \citep{dpok} and DDPO \citep{ddpo} apply policy gradient to the sampling process of diffusion models modeled by Markov decision process. However, the major drawback reward over-optimization \citep{test, hacking} which may harm generation quality and diversity. Then, impacted by the success of Direct Preference Optimization (DPO) \citep{dpo} in large language models (LLMs), Diffusion-DPO \citep{diffusiondpo, d3po} and D3PO adopt denoising steps of diffusion models to perform DPO. Furthermore, several recent works adjust the DPO objective with the essential characteristics of diffusion models. DSPO \citep{dspo} modifies the time-dependent score matching objectives to distill the score function of human-preferred image distributions into pretrained score functions. InPO \cite{inpo} and SmPO-Diffusion \citep{smooth} applies DDIM inversion technique in response to the challenge of the implicit rewards allocation in the long-chain denoising process. Diffusion-NPO \citep{npo} and Self-NPO \citep{self} address the efficacy of classifier-free guidance (CFG) and train a model attuned to negative preferences to bias away from the negative-conditional inputs.

Another line of work focuses on test-time alignment. DOODL \citep{doodl} directly  optimizes the diffusion latents at each timestep through the backpropagation of the reward model. \cite{demon} seeks to synthesize theocratically optimal noises based on multiple evaluations. DNO \cite{dno} turns to optimize the initial noise and develops a zeroth-order optimization algorithm for non-differential rewards. Unfortunately, all these improvements come with significant overload. Due to the special property of score distillation, our work provides a new perspective for preference alignment.

\subsection{Text-to-3D Generation}
\label{background}
The area this work belongs is distilling 2D into 3D. Based on the success of text-to-image diffusion models, Score Distillation Sampling (SDS) \citep{sds} was first proposed to distill a pretrained diffusion model $\boldsymbol{\epsilon}_\phi$ to generate 3D assets. Instead of guiding the optimization of differentiable 3D shape representation with PF-ODE (SDE) in the main paper, SDS aims to find modes of the score functions across all noise levels. Following the notations in the main paper, it can be expressed as updating the 3D representation with

\begin{equation}
\begin{aligned}
    \nabla_{\theta}\mathcal{L}_{\text{SDS}}(\theta) &= \nabla_{\theta}\mathbb{E}_{t,\boldsymbol{c},\boldsymbol{\epsilon}}\left[\sigma_{t}/\alpha_{t}w(t)\text{KL}\left(q(\mathbf{z}_{t}|g_{\theta}(\boldsymbol{c}); y, t) \| p_{\phi}(\mathbf{z}_{t}; y, t)\right)\right] \\
    &=  \mathbb{E}_{t,\boldsymbol{c},\boldsymbol{\epsilon}} \left[ w(t) (\boldsymbol{\epsilon}_\phi(\boldsymbol{x}_t,y,t) - \boldsymbol{\epsilon}) \frac{\partial g_{\theta}(\boldsymbol{c})}{\partial \theta} \right].
\end{aligned}
\label{sds}
\end{equation}

Many followup works devote to improve the behavior of SDS from many aspects, including improving view-consistency (multi-face Janus problem) via introducing stronger diffusion priors \citep{jointdreamer, richdreamer}, boosting geometry quality through coarse-to-fine training \citep{fantasia,magic3d,DreamMesh}, and accelerating generation process by applying more advanced 3D representation \citep{tet,3dgs1} or parallelization \citep{DreamPropeller}. Moreover, despite advancing technically, several recent works build connections
between score distillation and PF-ODE through deterministic noising \citep{consistent3d,cfd,sdi,fsd} or consistency training \citep{connect3d, segment}. Comparing with mode-seeking objective in Eq. \ref{sds}, these methods yield significant improvement on fidelity and diversity. In this perspective, our work aims to seek minimum conflicts to these advancements while aligning with preference, but for other relevant existing works, they basically only consider the derivation from Eq. \ref{sds}. DreamReward \citep{dreamreward,dreamrewardx} fine-tunes a reward model from ImageReward \citep{imagereward} to approximate the shift towards an ideal noise prediction network aligned with human preference. DreamAlign \citep{dreamalign} trains a reference noise prediction network using proposed D-3DPO algorithm and derive a preference contrastive loss. DreamDPO \citep{dreamdpo} completely eliminate the use of reference model, but we find it unstable and will do harm to fidelity. Concurrent work DreamCS \citep{dreamcs} address the geometry alignment issue through training a new reward model, but it is still under the framework of DreamReward. 

Despite generating 3D assets through distilling 2D diffusion priors, other prevailing paradigms such as leveraging large reconstruction models \citep{lrm,lgm} or native 3D generation \citep{trellis} also shows attractive performance especially on speed and geometry. However, due to the lack of high quality data and expensive computational demands, preference alignment have barely been explored in these fields.

\section{Derivation Details of Preference Score Guidance}
\label{derivation}

\subsection{From Eq.\ref{preference} to Eq. \ref{objective}}
To get the expansion in Eq. \ref{objective}, the core is to handle  term $\nabla_{x_t}\left(r\left(y, \boldsymbol{x}_t \right)-r\left(y, \boldsymbol{x}_t' \right)\right)$. Our solution is to use the implicit reward rewritten by the global optimal solution $p_{\phi}^*$ in Eq. \ref{reward}, which is 

\begin{equation}
    \nabla_{\boldsymbol{x}_t}(\Delta r_t)=\nabla_{\boldsymbol{x}_t}\left(r\left(y, \boldsymbol{x}_t^w \right)-r\left(y, \boldsymbol{x}_t^l \right)\right)=\beta\nabla_{\boldsymbol{x}_t} ( \log \frac{p^*_{\phi}(\boldsymbol{x}_{t}^w|y)}{q_{\theta}(\boldsymbol{x}_{t}^w|\boldsymbol{x}_{c}=g_{\theta}(\boldsymbol{c}))} -\log \frac{p^*_{\phi}(\boldsymbol{x}_{t}^l|y)}{q_{\theta}(\boldsymbol{x}_{t}^l|\boldsymbol{x}_{c}=g_{\theta}(\boldsymbol{c}))}).
    \label{grad}
\end{equation}
 The tricky part of the above expression is the terms related to $q_{\theta}$. Rigorously, since we construct the noising samples by adding noise, they can be expressed with terms related to noise. However, as mentioned in many previous works \citep{vsd,asd,cfd}, it is problematic because the rendering images especially at the early timesteps are out of the distribution of natural images. In our case, it will result in an inaccurate reference distribution. Therefore, following the practice of \citep{dreamdpo}, we neglect them in our final ODE. Eventually, if we approximate $p_{\phi}^*$ with $p_{\phi}$ for each step,  then put everything together in Eq. \ref{preference}, we have
\begin{equation}
    \begin{aligned}
        \mathrm{d}(\frac{\boldsymbol{x}_t}{\alpha_t}) &=\mathrm{d}(\frac{\sigma_t}{\alpha_t})\cdot(\boldsymbol{\epsilon}_\phi(\boldsymbol{x}_t,y,t)-\sigma_t(1-\sigma( \Delta r_t )\nabla_{\boldsymbol{x}_t}(\Delta r_t ))\\
        &\approx\mathrm{d}(\frac{\sigma_t}{\alpha_t})(\boldsymbol{\epsilon}_\phi(\boldsymbol{x}_t,y,t)-\beta\sigma_t \sigma( -\Delta r_t)(-\frac{\boldsymbol{\epsilon}_\phi(\boldsymbol{x}_t^{\text{w}},y,t)}{\sigma_t}-(-\frac{\boldsymbol{\epsilon}_\phi(\boldsymbol{x}_t^{\text{l}},y,t)}{\sigma_t})))\\
        &=\underbrace{\mathrm{d}(\frac{\sigma_t}{\alpha_t})}_{-lr}\cdot\underbrace{(\boldsymbol{\epsilon}_\phi(\boldsymbol{x}_t,y,t)+\beta\sigma(- \Delta r_t)(\boldsymbol{\epsilon}_\phi(\boldsymbol{x}_t^{\text{w}} ,y,t)-\boldsymbol{\epsilon}_\phi(\boldsymbol{x}_t^{\text{l}} ,y,t)))}_{\nabla\mathcal{L}}.
    \end{aligned}
\end{equation}

Mechanistically speaking,  our preference score guidance perfectly align with DPO \citep{dpo} where gradients are in the direction of increasing the likelihood of preferred samples, but more importantly, we successfully build connection between preference and guidance so that our PSD is performed without additional training. Also, thanks to the flexibility of score distillation, we are able to get access to different noise at each denoising step, which enables to implement the crucial win-lose score prediction in our preference score guidance.

\textbf{Change-of-variable technique.} Considering the chain rule in Eq. \ref{objective}, it is intuitive to replace $\boldsymbol{x}_t$ with $g_{\theta}(\boldsymbol{c})$ and directly use $\nabla\mathcal{L}$ to update 3D representation. However, as discussed in many previous works \citep{sdi, cfd, consistent3d}, since the rendering images are non-noisy, it will suffer form out-of-distribution issue. Therefore, they develop several noising techniques including fixed noise,  DDIM inversion and integral noise. Discussing these noising techniques is out of the scope of this paper, but it is necessary to point out the change-of-variable techniques in our PSD is applied to Eq. \ref{insight} and results in affecting the unconditional term so that our objective Eq. \ref{objective} can be formulated in a general form without the need of further discussion. Meanwhile, based on the above discussion, a great advantage of our PSD is it can combine with these noising techniques seamlessly which has been shown previously in the ablation of preference score for 2D image generation. For 3D generation, we test with integral noise \citep{kwak2024geometry,integral} modified by CFD \citep{cfd} in the 2-stage NeRF and 3-stage DMTet generation pipeline where wining noise (one of $\boldsymbol{\epsilon},\boldsymbol{\epsilon}'$) is replaced with

\begin{equation}
G(\mathbf{p}) = \frac{1}{\sqrt{|\Omega_{\mathbf{p}}|}} \sum_{A_i \in \Omega_{\mathbf{p}}} W(A_i).
\end{equation}
For the query pixel $\mathbf{p}$, $\Omega_{\mathbf{p}} = \mathcal{T}^{-1}(ctw_c(\mathbf{p}))$ is covered by after $\mathbf{p}$ is warped to a pre-set constant reference noise. For more details of this algorithm, interested readers may refer to the original paper of CFD \citep{cfd}, what's important is our experiments shown in Tab. \ref{complex} proves our compatibility with these nosing techniques. 

\subsection{Justification of Definition in Eq. \ref{rlhf_ours}}
\label{justification}

Our derivation heavily rely on the definition in Eq. \ref{rlhf_ours}. Therefore, discussing its validity is of first priority. We justify it by showing two representative works can also be related to this definition. 

\textbf{DreamReward \citep{dreamreward}:}
\begin{equation}
    \nabla_{\theta} \mathcal{L}_{\text{DreamReward}}(\theta) = \mathbb{E}_{t,\boldsymbol{c},\boldsymbol{\epsilon}} \left[ \omega(t) \left( \epsilon_{\phi}(\boldsymbol{x}_{t}, y, t) - \lambda_{r} \frac{\partial r_{Reward3D}(y, \hat{\boldsymbol{x}}_{0}, \boldsymbol{c})}{\partial g_{\theta}(\boldsymbol{c}))} - \epsilon \right) \frac{\partial g_{\theta}(\boldsymbol{c})}{\partial \theta}  \right].
    \label{dreamreward}
\end{equation}

DreamReward fine-tunes a view-dependent reward model Reward3D which approximates the difference between optimal diffusion model $\boldsymbol{\epsilon}_\phi^*$ and current diffusion model $\boldsymbol{\epsilon}_\phi$ with
\begin{equation}
    \boldsymbol{\epsilon}_\phi^*(\boldsymbol{x}_t,y,t) )-\boldsymbol{\epsilon}_\phi(\boldsymbol{x}_t,y,t) )=\frac{\partial r_{Reward3D}(\boldsymbol{x}_t,y,t)}{\partial g_{\theta}(\boldsymbol{c})}=-\frac{\partial r(\boldsymbol{x}_t,y,t)}{\partial g_{\theta}(\boldsymbol{c})}.
\end{equation}

Note that the purpose of our definition in Eq. \ref{rlhf_ours} is to formulate the preference guidance through implicit reward based on the connection between score distillation and diffusion process built by recent works \citep{kwak2024geometry, cfd, sdi}. However, DreamReward follows the framework of SDS, so in order to directly formulate a objective to optimize 3D representation, we make a revision to Eq. \ref{rlhf_ours}:

\begin{equation}
    \max_{\theta} \mathbb{E}_{t,\boldsymbol{c},\boldsymbol{\epsilon}} [r(y, \boldsymbol{x}_{t})] - \beta \mathbb{D}_{\text{KL}} [p_{\phi}(\boldsymbol{x}_{t}|y)||q_{\theta}(\boldsymbol{x}_{t}|\boldsymbol{x}_{0}=g_{\theta}(\boldsymbol{c})) ],
\end{equation}

Difference is the expectation we take is with respect to the added noise and parameters we tune is 3D representation $\theta$. This modification enables the objective to seek modes of the score functions, which fits in the framework of original SDS. After that, calculating its gradients directly (apply Sticking-the-Landing \citep{strike} type gradient suggested in \citep{sds} and set $\lambda_r=1/\beta$) leads to

\begin{equation}
    \begin{aligned}
     \nabla_{\theta}&\mathbb{E}_{t} [r(y, \boldsymbol{x}_{t})] - \beta\nabla_{\theta} \mathbb{D}_{\text{KL}} [p_{\phi}(\boldsymbol{x}_{t}|y)||q_{\theta}(\boldsymbol{x}_{t}|\boldsymbol{x}_{0}=g_{\theta}(\boldsymbol{c})) ] \\
     &=\mathbb{E}_{t} [\frac{\partial r(y, \boldsymbol{x}_{t})}{\partial g_{\theta}(\boldsymbol{c})}\frac{\partial g_{\theta}(\boldsymbol{c})}{\partial \theta}]- \beta \mathbb{E}_{t,\boldsymbol{c}} \left[ w(t) (\boldsymbol{\epsilon}-\boldsymbol{\epsilon}_\phi(\boldsymbol{x}_t,y,t) ) \frac{\partial g_{\theta}(\boldsymbol{c})}{\partial \theta} \right] \\
     &=\mathbb{E}_{t,\boldsymbol{c}} \left[ w(t) (\boldsymbol{\epsilon}_\phi(\boldsymbol{x}_t,y,t)+\lambda_r\frac{\partial r(y, \boldsymbol{x}_{t})}{\partial g_{\theta}(\boldsymbol{c})} -\boldsymbol{\epsilon}) \frac{\partial g_{\theta}(\boldsymbol{c})}{\partial \theta} \right]\\
     &=\mathbb{E}_{t,\boldsymbol{c}} \left[ w(t) (\boldsymbol{\epsilon}_\phi(\boldsymbol{x}_t,y,t)-\lambda_r\frac{\partial r_{Reward3D}(\boldsymbol{x}_t,y,t)}{\partial g_{\theta}(\boldsymbol{c})} -\boldsymbol{\epsilon}) \frac{\partial g_{\theta}(\boldsymbol{c})}{\partial \theta} \right].
     \label{dreamreward_ours}
    \end{aligned}
\end{equation}

\textbf{DreamDPO \citep{dreamdpo}:}
\begin{equation}
    \nabla_{\theta} \mathcal{L}_{\text{DreamDPO}}(\theta) = \mathbb{E}_{t,\boldsymbol{c}} \left[ w(t) \left( (\epsilon_{\phi}(\boldsymbol{x}_{t}^{\text{w}}, y, t) - \epsilon^{\text{w}}) - (\epsilon_{\phi}(\boldsymbol{x}_{t}^{\text{l}}, y, t) - \epsilon^{\text{l}}) \right) \frac{\partial g_{\theta}(\boldsymbol{c})}{\partial \theta} \right]
\label{dreamdpo}
\end{equation}

We ignore the hyperparameter $\tau$ in \citep{dreamdpo} since it is used to resolve numerical instability. Similar to the idea of Diffusion-DPO \citep{diffusiondpo}, $r\left(y, \boldsymbol{x}_t \right)$ can also be parameterized by a neural network $\Phi$ and estimated via maximum likelihood training for binary classification 

\begin{equation}
    \mathcal{L}_{\text{BT}}(\Phi) = -\mathbb{E}_{y, \boldsymbol{x}_{t}^{\text{win}}, \boldsymbol{x}_{t}^{\text{lose}}} \left[ \log \sigma \left( r_{\Phi}(y, \boldsymbol{x}_{t}^{\text{win}}) - r_{\Phi}(y, \boldsymbol{x}_{t}^{\text{lose}}) \right) \right],
\end{equation}
and plug the implicit reward rewritten by the global optimal solution $p_{\phi}^*$ in Eq. \ref{reward}, so that we get

\begin{equation}
\begin{aligned}
    \nabla_{\theta}\mathcal{L}_{\text{BT}}(\theta)&=-\mathbb{E}_{y, \boldsymbol{x}_{t}^{\text{w}}, \boldsymbol{x}_{t}^{\text{l}}}[\sigma\left( r\left(y, \boldsymbol{x}_t^{\text{l}} \right) - r\left(y, \boldsymbol{x}_t^{\text{w}} \right)\right) \\
    &\cdot\beta\nabla_{\theta} ( \log \frac{p^*_{\phi}(\boldsymbol{x}_{t}^w|y)}{q_{\theta}(\boldsymbol{x}_{t}^w|\boldsymbol{x}_{c}=g_{\theta}(\boldsymbol{c}))} -\log \frac{p^*_{\phi}(\boldsymbol{x}_{t}^l|y)}{q_{\theta}(\boldsymbol{x}_{t}^l|\boldsymbol{x}_{c}=g_{\theta}(\boldsymbol{c}))})]\\
    &=-\mathbb{E}_{y, \boldsymbol{x}_{t}^{\text{w}}, \boldsymbol{x}_{t}^{\text{l}}}[\sigma( r(y, \boldsymbol{x}_t^{\text{l}} ) - r(y, \boldsymbol{x}_t^{\text{w}} ))(-\frac{\boldsymbol{\epsilon}_\phi(\boldsymbol{x}_t^{\text{w}},y,t)}{\sigma_t}-(-\frac{\boldsymbol{\epsilon}_\phi(\boldsymbol{x}_t^{\text{l}},y,t)}{\sigma_t}))]\\
    &=\mathbb{E}_{y, \boldsymbol{x}_{t}^{\text{w}}, \boldsymbol{x}_{t}^{\text{l}}}[\sigma( r(y, \boldsymbol{x}_t^{\text{l}} ) - r(y, \boldsymbol{x}_t^{\text{w}} ))(\boldsymbol{\epsilon}_\phi(\boldsymbol{x}_t^{\text{w}},y,t)-\boldsymbol{\epsilon}_\phi(\boldsymbol{x}_t^{\text{l}},y,t))].
\end{aligned}
\end{equation}
Then, as suggested by \citep{sds}, if we use Sticking-the-Landing type gradient to control variate by keeping the noise added in Eq. \ref{win-lose}, we have an objective similar to Eq. \ref{dreamdpo}. The only difference is the weight $\sigma( r(y, \boldsymbol{x}_t^{\text{l}} ) - r(y, \boldsymbol{x}_t^{\text{w}} ))$.

\section{Theoretical Analysis of Negative Embedding Update Strategy}
\label{appendix_proof}

In the main paper, we have claimed the motivation of our negative embedding update strategy is to approach the optimal score function derived directly from our specific RLHF formulation Eq. \ref{rlhf_ours}. In this section, we provide a mathematical analysis of it. Although the following proof is not completely rigorous, it can still demonstrate the efficacy and self-consistency of our approach.

\begin{proposition}
\label{prop:neg_opt}
Let $p_\phi(\boldsymbol{x}_t|y)$ be the distribution of a frozen pretrained diffusion model and $r(\boldsymbol{x})$ be a differentiable reward function. Under our definition Eq. \ref{rlhf_ours}, optimizing the negative embedding $n$ to maximize the reward $r(\hat{\boldsymbol{x}}_0)$ is approximately equivalent to minimizing the Fisher Divergence between the effective score function induced by CFG and the optimal score function of the reward-tilted distribution.
\end{proposition}

\begin{proof}
\textbf{1. The Optimal Score Function.}
Following standard RLHF in diffusion models, the optimal human-aligned distribution $p^*(\boldsymbol{x}|y)$ as the pretrained distribution tilted by the reward function with inverse temperature $\beta$ is
\begin{equation}
\label{eq:proof_opt_p_star}
    p^*(\boldsymbol{x}_t | y) \propto q_\theta(\boldsymbol{x}_t) \exp\left( \beta r(y, \boldsymbol{x}_t) \right).
\end{equation}
The score function $\nabla_{\boldsymbol{x}_t} \log p^*(\boldsymbol{x}_t | y)$ decomposes into an (unknown) noise and the reward gradient. Similar to the dilemma we have encountered in \ref{grad}, we cannot express term $q_\theta(\boldsymbol{x}_t)$ precisely. However, since we have a pretrained diffusion model, we can approximate it using the noise prediction. Therefore, the optimal noise prediction $\boldsymbol{\epsilon}^*$ is:
\begin{equation}
\label{eq:optimal_eps}
    \begin{aligned}
    \boldsymbol{\epsilon}^*(\boldsymbol{x}_t, y) &= -\sigma_t \nabla_{\boldsymbol{x}_t} \log p^*(\boldsymbol{x}_t | y) \\
    &= -\sigma_t \left( \nabla_{\boldsymbol{x}_t} \log q_{\theta}[(\boldsymbol{x}_t | y) + \beta \nabla_{\boldsymbol{x}_t} r(\boldsymbol{x}_t) \right) \\
    &\approx \boldsymbol{\epsilon}_\phi(\boldsymbol{x}_t, y, t) - \sigma_t \beta \nabla_{\boldsymbol{x}_t} r(\boldsymbol{x}_t).
    \end{aligned}
\end{equation}
Eq. \ref{eq:optimal_eps} indicates that to sample from the optimal distribution, the noise prediction must shift against the direction of the reward gradient.

\textbf{2. The CFG Approximation.}
With a frozen model parameter $\phi$, we employ CFG with a learnable negative embedding $n$. The effective noise output $\tilde{\boldsymbol{\epsilon}}$ is given by:
\begin{equation}
\label{eq:cfg_def}
    \tilde{\boldsymbol{\epsilon}}(\boldsymbol{x}_t, y, n) = \boldsymbol{\epsilon}_\phi(\boldsymbol{x}_t, y) + \gamma \left( \boldsymbol{\epsilon}_\phi(\boldsymbol{x}_t, y) - \boldsymbol{\epsilon}_\phi(\boldsymbol{x}_t, n) \right).
\end{equation}
To minimize the the Fisher Divergence, we have to enforce $\tilde{\boldsymbol{\epsilon}} = \boldsymbol{\epsilon}^*$, so we equate Eq. \ref{eq:cfg_def} and Eq. \ref{eq:optimal_eps}:
\begin{equation}
    \boldsymbol{\epsilon}_\phi(y) + \gamma (\boldsymbol{\epsilon}_\phi(y) - \boldsymbol{\epsilon}_\phi(n)) \approx \boldsymbol{\epsilon}_\phi(y) - \sigma_t \beta \nabla_{\boldsymbol{x}_t} r(\boldsymbol{x}_t).
\end{equation}
Solving for the term involving the negative embedding $\boldsymbol{\epsilon}_\phi(\boldsymbol{x}_t, n)$:
\begin{equation}
    \gamma (\boldsymbol{\epsilon}_\phi(y) - \boldsymbol{\epsilon}_\phi(n)) \approx - \sigma_t \beta \nabla_{\boldsymbol{x}_t} r(\boldsymbol{x}_t)
\end{equation}
\begin{equation}
\label{eq:target_condition}
    \boldsymbol{\epsilon}_\phi(\boldsymbol{x}_t, n) \approx \boldsymbol{\epsilon}_\phi(\boldsymbol{x}_t, y) + \frac{\sigma_t \beta}{\gamma} \nabla_{\boldsymbol{x}_t} r(\boldsymbol{x}_t).
\end{equation}
Eq.~\ref{eq:target_condition} reveals the necessary condition: for the CFG output to match the optimal score, the negative embedding $n$ must induce a noise prediction that aligns with the \textit{positive} direction of the reward gradient relative to the positive prompt.

\textbf{3. Gradient Analysis of Our Optimization.}
We now analyze the gradient direction used in our method (Algorithm 1). We maximize the reward $J(n) = r(\hat{\boldsymbol{x}}_0(n))$ with respect to $n$. Using Tweedie's formula $\hat{\boldsymbol{x}}_0 = (\boldsymbol{x}_t - \sigma_t \tilde{\boldsymbol{\epsilon}})/\alpha_t$, we apply the chain rule:
\begin{equation}
    \nabla_n J(n) = \left( \nabla_{\hat{\boldsymbol{x}}_0} r \right)^\top \frac{\partial \hat{\boldsymbol{x}}_0}{\partial \tilde{\boldsymbol{\epsilon}}} \frac{\partial \tilde{\boldsymbol{\epsilon}}}{\partial \boldsymbol{\epsilon}_\phi(n)} \frac{\partial \boldsymbol{\epsilon}_\phi(n)}{\partial n}.
\end{equation}
Substituting the partial derivatives:
\begin{itemize}
    \item From Tweedie's formula: $\frac{\partial \hat{\boldsymbol{x}}_0}{\partial \tilde{\boldsymbol{\epsilon}}} = -\frac{\sigma_t}{\alpha_t} \mathbf{I}$.
    \item From CFG definition (Eq.~\ref{eq:cfg_def}): $\frac{\partial \tilde{\boldsymbol{\epsilon}}}{\partial \boldsymbol{\epsilon}_\phi(n)} = -\gamma \mathbf{I}$.
\end{itemize}
Substituting these back into the gradient equation:
\begin{equation}
\label{eq:final_grad}
    \begin{aligned}
    \nabla_n J(n) &= \left( \nabla_{\hat{\boldsymbol{x}}_0} r \right)^\top \left( -\frac{\sigma_t}{\alpha_t} \right) \cdot (-\gamma) \cdot \nabla_n \boldsymbol{\epsilon}_\phi(n) \\
    &= \frac{\gamma \sigma_t}{\alpha_t} \left[ \left( \nabla_{\hat{\boldsymbol{x}}_0} r \right)^\top \nabla_n \boldsymbol{\epsilon}_\phi(n) \right].
    \end{aligned}
\end{equation}
Since $\alpha_t, \sigma_t, \gamma > 0$, the coefficient is positive. This implies that performing gradient ascent $n \leftarrow n + \eta \nabla_n J(n)$ updates $n$ such that the output $\boldsymbol{\epsilon}_\phi(\boldsymbol{x}_t, n)$ moves in the direction of $\nabla_{\hat{\boldsymbol{x}}_0} r$.
\end{proof}

\section{Comparisons with Existing Methods}
\label{comparision}

\textbf{Comparison with SDS.} Although our derivation is from the perspective of PF-ODE in Eq. \ref{ode}, as already discussed in many previous paper \cite{sdi,cfd,fsd,dreamsampler} SDS can be treated as a special case. Therefore, Eq. \ref{objective} can cover all these variants. The main difference is we derive an additional preference score guidance while SDS only consists $\delta_{gen}$ and $\delta_{cls}$. Also, in Eq. \ref{grad}, we don't additionally apply a Sticking-the-Landing type gradient since it is only a guidance.

\textbf{Comparison with DreamReward.} In Eq. \ref{dreamreward_ours}, DreamReward can be interpenetrated as the gradient of our defined RLHF objective , which implies it's also introducing additional guidance to the generation process. In the formulation of ODE, it can be not strictly written as (omit time-dependent coefficients as well)

\begin{equation}
\mathrm{d}(\frac{\boldsymbol{x}_t}{\alpha_t}) =\mathrm{d}(\frac{\sigma_t}{\alpha_t})\cdot\left(\boldsymbol{\epsilon}_\phi(\boldsymbol{x}_t,y,t)-\nabla_{\boldsymbol{x}_t}r_{Reward3D}(y, \boldsymbol{x}_t, \boldsymbol{c})\right).
\label{ode_dreamreward}
\end{equation}

Obviously, comparing Eq. \ref{ode_dreamreward} with classifier guidance in \citep{image1}, the reward model provides extra guidance that requires  pixel-wise gradient directly operating on $\boldsymbol{x}_t$. This enforces retraining of the reward model and is defective when a) 3d data is rare, b) intermediate noisy steps. In contrast, our PSD overcomes this issue completely. At each timestep, two terms in preference score guidance are both the output of the pretrained diffusion models, and the negative embedding optimization strategy also avoids directly operating on $\boldsymbol{x}_t$. Consequently, our results presents much less artifacts. In addition, this perspective directly reveal the source of reward hacking in DreamReward.

\textbf{Comparison with DreamDPO.} Apart from the mainstream derivation we present in the main paper, the novelty of our approach can be supported by another intuition given by the connection between CSD \citep{csd} and DreamDPO under our framework. While comparing with DreamReward, we can easily write Eq.\ref{ode_dreamreward}, but for DreamDPO, its connection with denoising trajectory is ambiguous. The breakthrough is from the motivation of CSD. CSD notices SDS is heavily relied on high CFG value, so they investigate and discover that using term $\delta_{cls}$ ($p(y|\boldsymbol{x}_t)$) alone is enough to generate 3D assets. While under our derivation, DreamDPO is actually $p\left(\mathcal{S}_{\text{pref}} \mid \boldsymbol{x}_t, y \right)$, so this explains the mechanism of DreamDPO in another important perspective. 

\begin{figure}[h] 
\centering 
\includegraphics[width=0.9\textwidth]{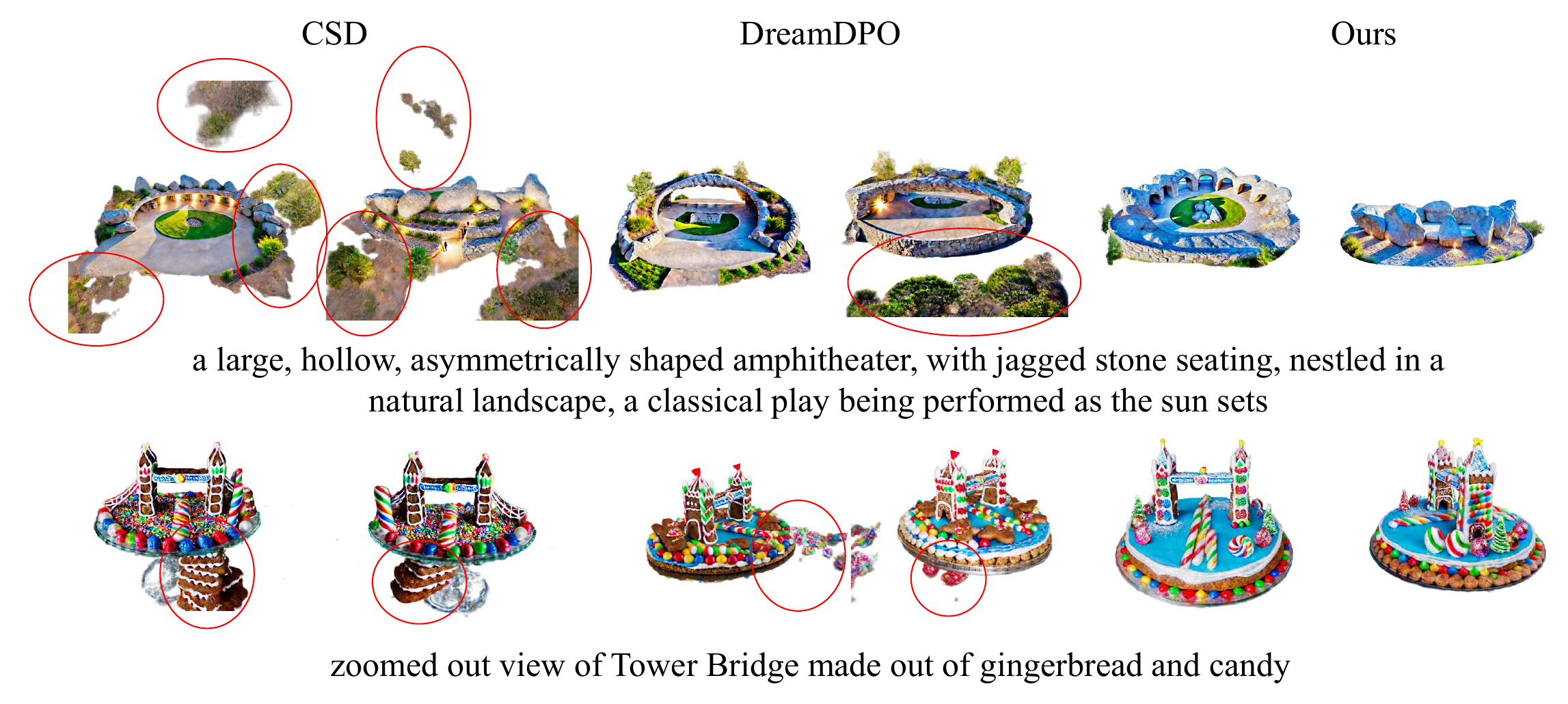} 
\caption{Comaprsion with CSD \citep{csd} and DreamDPO \citep{dreamdpo}. Artifacts are marked with red circles.} 
\label{csd} 
\end{figure}

To verify this claim, we perform a simple comparison. A major drawbacks of CSD is it will also result in artifacts \citep{csd,bridge}. Shown in Fig. \ref{csd}, CSD and DreamDPO produces similar pattern of artifacts, while our PSD behaves normally. Therefore, based on the above analysis, our work is fundamentally different from DreamDPO even not considering the negative embedding optimization strategy.

\textbf{Comparison with other works applying prompt embedding optimization.} There has been several previous also applying prompt embedding update but for completely different reasons. DiverseDream \citep{diverse} employ HiPer token inversion to augment diversity and update the last several token of prompt embedding $y$ to achieve similar effect of VSD \cite{vsd} in a memory-efficient way. LODS \citep{lods} also optimize null (negative) embedding, but in order to reduce CFG value via "normalizing" SDS. On the other side, we incorporate CFG to update negative embedding in order to improve network towards higher rewards, which is novel to existing methods.

\section{Implementation Details}
\label{details}

In this appendix, we describe the missing details of implementation in the main paper.

\subsection{Configuration}
\label{regularizer}

For fair comparison, we do not propose any new regularizer and maintain same configuration for all methods in each experiment.

\textbf{Single-stage distillation of MVDream.} The configuration of this experiment follows baseline method MVDream \citep{mvdream}, where only orientation loss proposed in \citep{sds} is used. The optimization takes 10000 steps with the weight of orientation loss linearly increasing from 10 to 1000 in the first 5000 steps. Training resolution is set to be $64\times64$ in first 5000 steps and $256\times256$ in the latter. For our negative embedding optimization strategy, learning rate is set as $1e^{-4}$ constantly. Due to memory limitation, Lambertian shading is not used in this setting. 

\textbf{2-stage NeRF generation.} The configuration of this experiment is adapted provided in  baseline method CFD \citep{cfd}, which is performed based on previous experiment. The optimization takes 20000 steps in resolution $512\times512$ with the weight of "z-varience" loss proposed in \citep{hifa} being 10 and normal smooth loss being 1000. In the latter stage, negative embedding is optimized with a linearly decreasing learning rate from $1e^{-4}$ to 0 in the first 1000 step.

\textbf{3-stage DMTet generation.} This experiment is also performed after single-stage distillation of MVDream. Geometry optimization takes 15000 steps and texture takes another 20000 steps, with both resolution being $1024\times1024$. Since reward models are trained in RGB space, we only PSD in the texture space with  learning rate being $1e^{-5}$ in the first 1000 step.

\subsection{Metrics}

\textbf{Human preference reward models.} For ImageReward \citep{imagereward}, PickScore \citep{pickscore}, Aesthetic scores \citep{laion}, and Multi-dimensional Preference Score \citep{mps} we evaluate the average of the scores across 60 equally spaced views and their corresponding text prompts.

\textbf{VQA models.} We apply Qwen2.5-VL-7B \citep{qwen} to calculate the text-3d alignment score. The question-answer pair we utilize is generated in Eval3d \citep{eval3d}. We evaluate across 12 renderings. For more details, please refer to \citep{eval3d}.

In order to compare geometry quality, we use the following metrics newly proposed from Eval3d:

\textbf{Geometric Consistency.} It measures the consistency between the surface normals analytically derived from the 3D representation and the normals predicted by a dense estimation model from 2D images. Our analytic normals are calculated using PyTorch auto-differentiation and estimated normals are calculated by converting from depth estimation from Depth Anything \citep{anything}, same as the original paper. The metric is calculated as follows:

\begin{equation}
    \text{Geometric consistency} = \frac{1}{N_p} \sum_{p} \mathbb{1}[\arccos(\mathbf{n}_p^{\text{anal}} \cdot \mathbf{n}_p^{\text{pred}}) < \delta^{\text{norm}}]
\end{equation}

where $N_p$ is number of valid pixel $p$, $\mathbb{1}(\cdot)$ is indicator function, $\mathbf{n}_p^{\text{anal}}, \mathbf{n}_p^{\text{pred}}$ are analytical and estimated normals respectively. $\delta^{\text{norm}}=23^{\circ}$ is a threshold.

\textbf{Semantic Consistency.} This metric measures the change of the underlying content and semantics. It projects 3D point $x$ to 2D DINO feature $\{\mathcal{F}_i^{\text{DINO}}$ of each rendered image via projection $\pi$ to retrieve its corresponding features. Then calculate

\begin{equation}  \label{eq:semantic_consistency}
    \text{Semantic consistency} = \frac{1}{N_{\text{vert}}} \sum_{x^{\text{vert}}} \mathbb{1}[\operatorname{mean}(\operatorname{Var}(\{\mathcal{F}_i^{\text{DINO}}(\pi_{v_i}(x^{\text{vert}}))\})) < \delta^{\text{DINO}}]
\end{equation}

where $x^{\text{vert}}$ is vertices of the 3D mesh. Following the original paper, $\delta^{\text{DINO}}$ is set to be the 70th percentile of average DINO variance.

\textbf{Structural Consistency.} Comparing with semantic consistency, structural consistency measures whether the overall structure of the generated asset is coherent and plausible. It uses  diffusion-based novel view synthesis model Stable-Zero123 \citep{zero} to predict the image at unobserved viewpoint. Then it applies perceptual metric DreamSim \citep{DreamSim} for similarity measurement. This metric can be formulated as 

\begin{equation}
\text{Structural consistency} = \max_{i} \frac{1}{N} \sum_{j=1}^{N} \left( 1 - f^{\text{DreamSim}}(\mathcal{I}_{i \to j}^{\text{pred}}, \mathcal{I}_{j}) \right)
\end{equation}
where $\mathcal{I}_{i \to j}$ is image predicted from viewpoint $i$.

For interested readers, please refer to the original paper of Eval3d \citep{eval3d} for more details. We choose these metrics because they are especially suitable to localize visual artifacts.

\begin{figure}[h] 
\centering 
\includegraphics[width=0.9\textwidth]{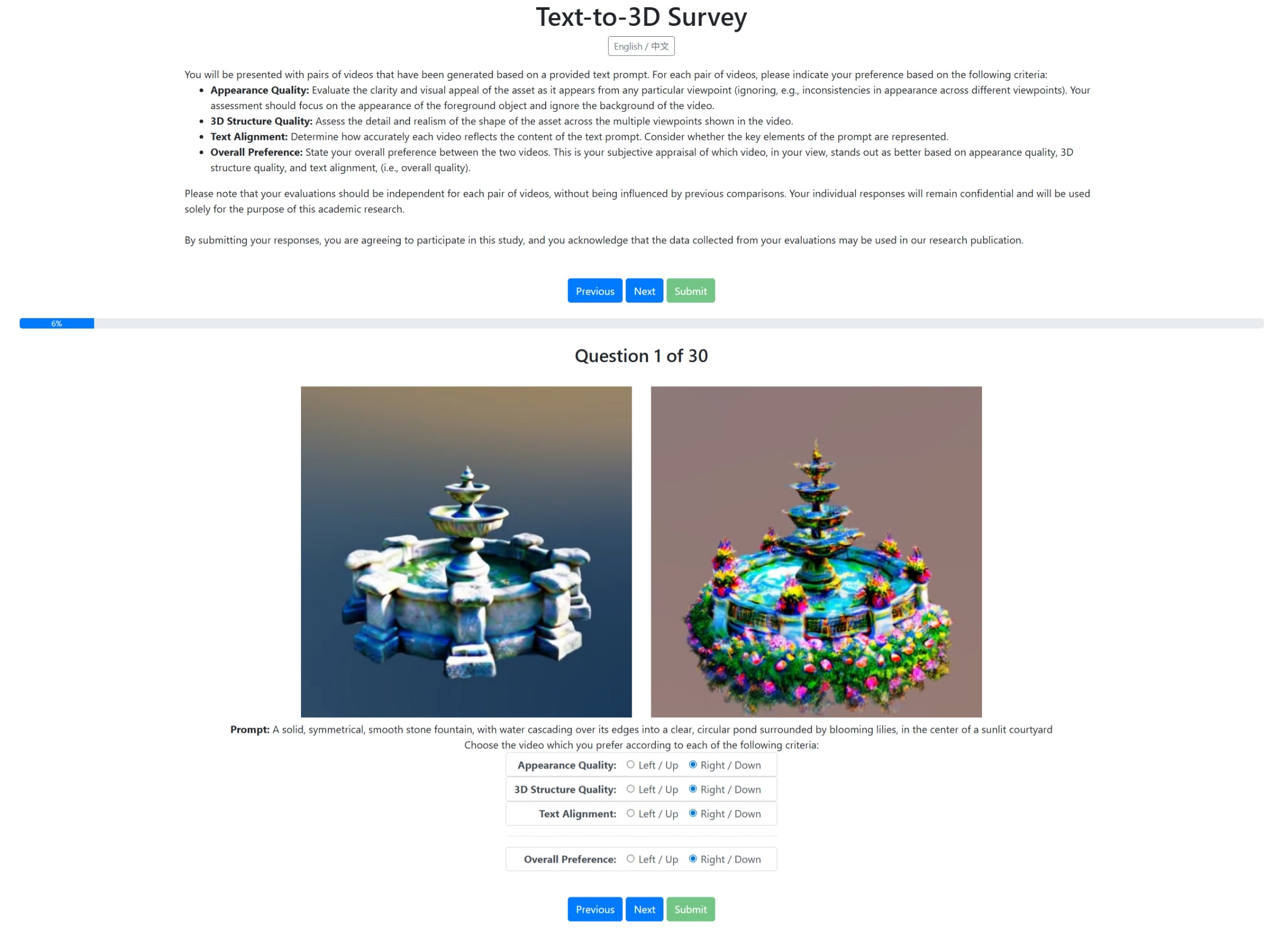} 
\caption{Screen shot of our survey web.} 
\label{screen} 
\end{figure}

\subsection{User Study \label{app_user}}

To validate that our results are truly preferred by human users, 23 participants are involved to make judgments over anonymous 30 rendered videos generated by our PSD against MVDream, DreamDPO and DreamReward. The instructions are:
\begin{itemize}
    \item Appearance Quality: Evaluate the clarity and visual appeal of the asset as it appears from any particular viewpoint (ignoring, e.g., inconsistencies in appearance across different viewpoints). Your assessment should focus on the appearance of the foreground object and ignore the background of the video.
    \item 3D Structure Quality: Assess the detail and realism of the shape of the asset across the multiple viewpoints shown in the video.
    \item Text Alignment: Determine how accurately each video reflects the content of the text prompt. Consider whether the key elements of the prompt are represented.
    \item Overall Preference: State your overall preference between the two videos. This is your subjective appraisal of which video, in your view, stands out as better based on appearance quality, 3D structure quality, and text alignment, (i.e., overall quality).
\end{itemize}
A screen shot of our survey web is shown in Fig. \ref{screen}.

\subsection{Algorithm}
In this section, we present the algorithm of PSD.

\begin{algorithm}[H]
\caption{Preference Score Distillation}
\small
\label{alg_psd}
\begin{algorithmic}[1]

\REQUIRE 
Prompt $y$; negative descriptors $y_{neg}$; pretrained text-to-image diffusion model $\boldsymbol{\epsilon}_{\phi}$; 
reward model $r$; learning rates $lr_1, lr_2$ for 3D representation $\theta$ and negative embeddings $n$; 
annealing schedule $t(\tau)$; classifier-free guidance weight $\gamma$.

\STATE Initialize 3D representation $\theta$ and negative embeddings $n$ using $y_{neg}$.

\WHILE{not converged}
    \STATE Randomly sample camera pose $\boldsymbol{c}$ and noise pairs 
    $\boldsymbol{\epsilon}^1, \boldsymbol{\epsilon}^2 \sim \mathcal{N}(\mathbf{0}, \mathbf{I})$.
    
    \STATE Render 3D representation $\theta$ at pose $\boldsymbol{c}$ to obtain 
    $\boldsymbol{x}_{c} = g_{\theta}(\boldsymbol{c})$.
    
    \STATE Compute diffusion timestep $t(\tau)$.
    
    \STATE Add noise to obtain noisy samples $\boldsymbol{x}_{t}^1$ and $\boldsymbol{x}_{t}^2$.
    
    \STATE Predict one-step samples $(\hat{\boldsymbol{x}}_{0}^1, \hat{\boldsymbol{x}}_{0}^2)$ and rank them using
    reward model $(r(y, \hat{\boldsymbol{x}}_{0}^1), r(y, \hat{\boldsymbol{x}}_{0}^2))$ to obtain
    winner $\boldsymbol{x}_t^w$ and loser $\boldsymbol{x}_t^l$ with corresponding noises
    $(\boldsymbol{\epsilon}, \boldsymbol{\epsilon}')$.
    
    \STATE Compute
    $\delta_{gen} \leftarrow \boldsymbol{\epsilon}_{\phi}(\boldsymbol{x}_t^w, t) - \boldsymbol{\epsilon}$,
    $\delta_{cls} \leftarrow \boldsymbol{\epsilon}_{\phi}(\boldsymbol{x}_t^w, y, t) -
    \boldsymbol{\epsilon}_{\phi}(\boldsymbol{x}_t^w, t)$,
    $\delta_{pref} \leftarrow \tilde{\boldsymbol{\epsilon}}_{\phi}(\boldsymbol{x}_t^w, y, t) -
    \tilde{\boldsymbol{\epsilon}}_{\phi}(\boldsymbol{x}_t^l, y, t)$.
    
    \STATE Compute
    \[
    \beta_r \leftarrow \gamma
    \frac{\|\delta_{cls}\|_2}{\|\delta_{pref}\|_2}
    \cdot
    \sigma\!\left(
    r(y, \hat{\boldsymbol{x}}_{0}^l) -
    r(y, \hat{\boldsymbol{x}}_{0}^w)
    \right)
    \]
    
    \STATE Update 3D representation:
    \[
    \theta \leftarrow \theta - lr_1 \cdot
    \mathbb{E}_{\boldsymbol{c}}
    \left[
    \left(
    \delta_{gen} + \gamma \delta_{cls} + \beta_r \delta_{pref}
    \right)
    \frac{\partial g_{\theta}(\boldsymbol{c})}{\partial \theta}
    \right]
    \]
    
    \STATE Update negative embeddings:
    \[
    n \leftarrow n - lr_2 \cdot
    \nabla_{n} \mathbb{E}_{\boldsymbol{c}}
    \left[
    -r(y, \hat{\boldsymbol{x}}_0)
    \right]
    \]
\ENDWHILE

\RETURN Optimized 3D representation $\theta$ and negative embeddings $n$.

\end{algorithmic}
\end{algorithm}

\subsection{Details of Toy Experiments on Image Generation}
\label{2d_details}
In the main paper, we set up a toy experiments to directly illustrate the effect of our proposed preference score guidance. To simulate the configuration in DDIM, we follow \citep{sdi} to perform a 50-step update of the image parameterized by $\theta$, as presented in Algorithm \ref{alg:2d}.

\section{Limitations}

Inheriting computational demands from prior score distillation methods, PSD optimization requires one to several hours per generation, limiting real-time applicability. While this work primarily investigates differentiable rewards, future work should explore non-differentiable objectives (e.g., vision-language model ensembles). Although PSD elevates aesthetic metrics, its performance remains bounded by the reward model’s capabilities, risking distribution shift or reward hacking. Moreover, as PSD’s output depends directly on the reward model, implementing content filtering modules is essential to prevent malicious content generation.

Additionally, we present failure cases of our proposed PSD in this appendix. While we successfully overcome the floating artifacts caused by pixel-level conflicts between reward gradients and diffusion dynamics, our method still fails to outperform  at certain cases. Besides, Janus problem still exists when distilling Stable Diffusion 2.1. Developing better rewards may become a new solution to this problem.

\begin{algorithm}[H]
\caption{2D Image Generation using Score Distillation}
\small
\label{alg:2d}
\begin{algorithmic}[1]

\REQUIRE 
Prompt $y$; pretrained text-to-image diffusion model $\boldsymbol{\epsilon}_{\phi}$;
learning rate $lr$; number of optimization steps $N$;
time shift interval $[T_1, T_2]$.

\STATE Initialize random latent $\boldsymbol{x}_T$ parameterized by $\theta \sim \mathcal{N}(\mathbf{0}, \mathbf{I})$.

\WHILE{not converged}
    \STATE Compute update timesteps $\{t_i\}_{i=1}^{N}$.
    
    \FOR{$i = 1$ to $N$}
        \STATE Sample noise $\boldsymbol{\epsilon}^2 \sim \mathcal{N}(\mathbf{0}, \mathbf{I})$
        and time shift $\tau \sim \mathcal{U}(T_1, T_2)$.
        
        \STATE Add noise to obtain noisy samples
        $\boldsymbol{x}_{t_i+\tau}^1$ and $\boldsymbol{x}_{t_i+\tau}^2$.
        
        \STATE Predict one-step samples
        $(\hat{\boldsymbol{x}}_{0}^1, \hat{\boldsymbol{x}}_{0}^2)$
        and rank them using reward scores
        $(r(y, \hat{\boldsymbol{x}}_{0}^1), r(y, \hat{\boldsymbol{x}}_{0}^2))$
        to obtain winner $\boldsymbol{x}_{t_i+\tau}^w$ and loser
        $\boldsymbol{x}_{t_i+\tau}^l$ with corresponding noises
        $(\boldsymbol{\epsilon}, \boldsymbol{\epsilon}')$.
        
        \STATE Compute
        \[
        \delta_{gen} \leftarrow
        \boldsymbol{\epsilon}_{\phi}(\boldsymbol{x}_{t_i+\tau}^w, t_i+\tau)
        - \boldsymbol{\epsilon}
        \]
        
        \STATE Compute
        \[
        \delta_{cls} \leftarrow
        \boldsymbol{\epsilon}_{\phi}(\boldsymbol{x}_{t_i+\tau}^w, y, t_i+\tau)
        -
        \boldsymbol{\epsilon}_{\phi}(\boldsymbol{x}_{t_i+\tau}^w, t_i+\tau)
        \]
        
        \STATE Compute
        \[
        \delta_{pref} \leftarrow
        \tilde{\boldsymbol{\epsilon}}_{\phi}(\boldsymbol{x}_{t_i+\tau}^w, y, t_i+\tau)
        -
        \tilde{\boldsymbol{\epsilon}}_{\phi}(\boldsymbol{x}_{t_i+\tau}^l, y, t_i+\tau)
        \]
        
        \STATE Compute
        \[
        \beta_r \leftarrow
        \gamma
        \frac{\|\delta_{cls}\|_2}{\|\delta_{pref}\|_2}
        \cdot
        \sigma\!\left(
        r(y, \hat{\boldsymbol{x}}_{0}^l) -
        r(y, \hat{\boldsymbol{x}}_{0}^w)
        \right)
        \]
        
        \STATE Update parameters:
        \[
        \theta \leftarrow \theta - lr \cdot
        \mathbb{E}
        \left[
        \left(
        \delta_{gen} + \gamma \delta_{cls} + \beta_r \delta_{pref}
        \right)
        \frac{\partial g_{\theta}(\boldsymbol{c})}{\partial \theta}
        \right]
        \]
    \ENDFOR
\ENDWHILE

\RETURN Optimized 2D image representation.

\end{algorithmic}
\end{algorithm}

\begin{figure}[h] 
\centering 
\includegraphics[width=0.9\textwidth]{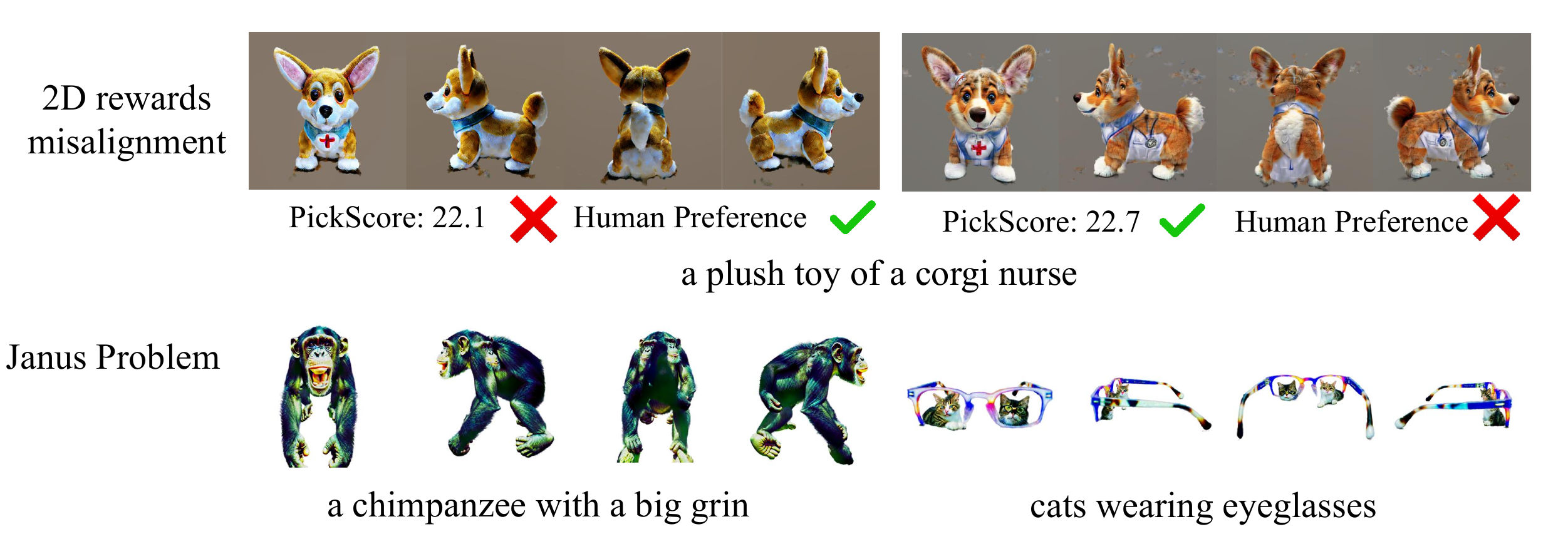} 
\caption{Visual examples of failure cases.}
\label{failure} 
\end{figure}

\subsection{Test Prompts}
\label{prompt}
For evaluations of single stage distillation of MVDream, we use 200-prompt croup set in Eval3D. For more complex 2-stage NeRF and 3-stage DMTet pipelines, we filter a harder subset consisting 40 prompts with lowest PickScore of MVDream, which is listed in Tab. \ref{tab:prompt}.

\begin{figure}[t] 
\centering 
\includegraphics[width=0.85\textwidth]{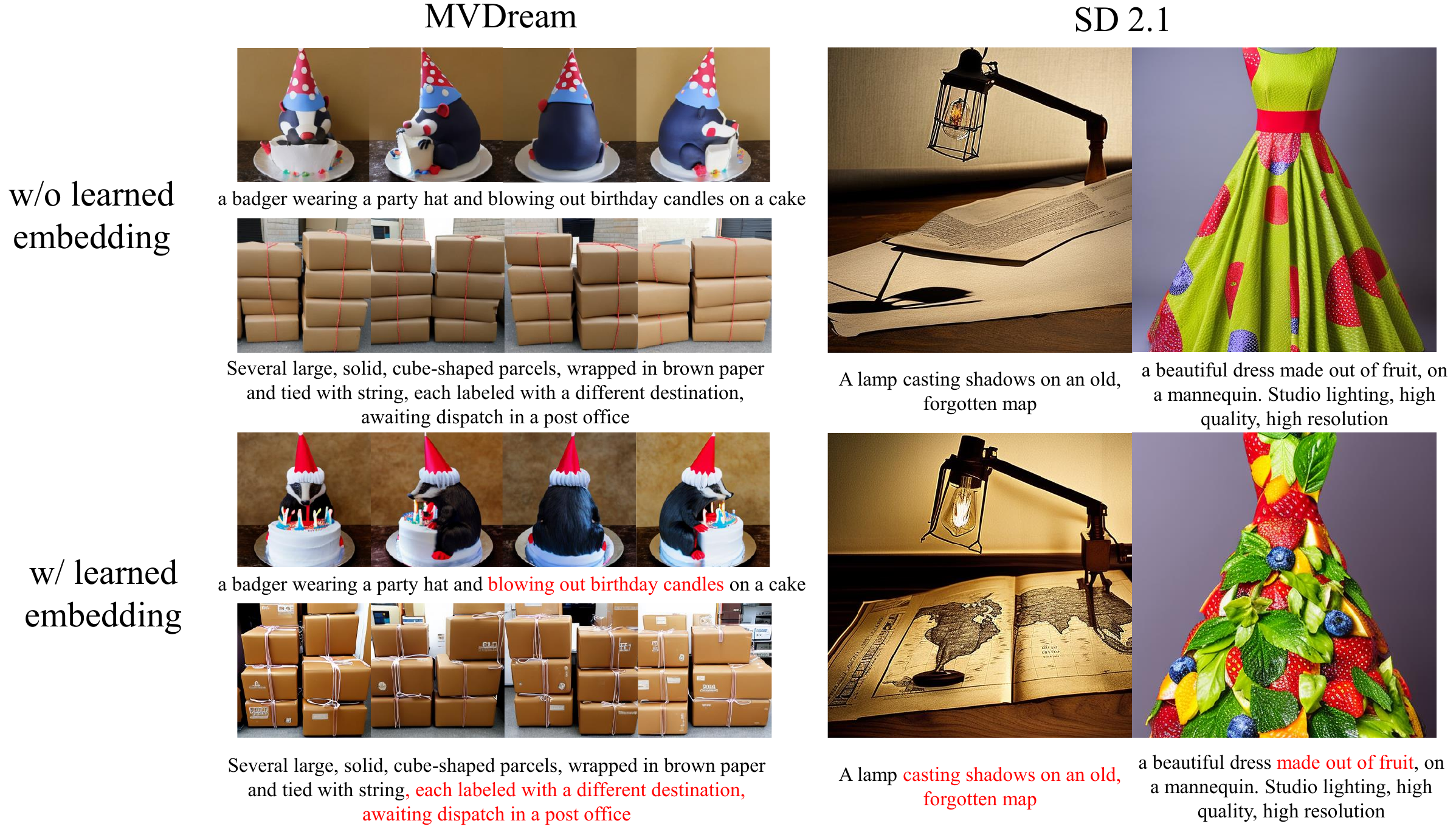} 
\caption{Visualization of learned negative embeddings on diffusion priors.}
\label{visual_neg} 
\end{figure}

\section{Supplementary Results and Comparisons}

\subsection{Visualization of learned negative embeddings on diffusion priors}

In this section, we conduct an experiment to visualize the difference of sampling diffusion priors with or without our learned negative embedding to illustrate the impact of learned negative embeddings on diffusion priors. Results are list in Tab. \ref{tab_neg} and visualized in Fig. \ref{visual_neg}. Note that although the learned negative embeddings we use here is are from the single stage distillation of MVDream, they are able to transfer to SD 2.1 because they use the same text encoder. This observation is consistent with ReNeg.

\begin{table}[h]
    \centering
    \small
    {
\caption{Impact of learned negative embeddings on diffusion priors. Metrics are evaluated across our 40-prompt subset on the average over 5 random seeds. \label{tab_neg} }
}
    \begin{tabular}{lcccc}
\toprule
 \multirow{2}*{\textbf{Experiments}} & \multicolumn{2}{c}{\textbf{MVDream}} & \multicolumn{2}{c}{\textbf{Stable Diffusion 2.1}}  \\
\cmidrule(lr){2-3} \cmidrule(lr){4-5} 
 & \textbf{Pick. $\uparrow$} & \textbf{I.R. $\uparrow$} & \textbf{Pick. $\uparrow$} & \textbf{I.R. $\uparrow$}   \\
\midrule
w/o $n$  & 20.23 & -0.31 & 21.07 & 0.30\\
w/ $n$ & \textbf{20.31} & \textbf{-0.11} & \textbf{21.12} & \textbf{0.41} \\
\bottomrule
\end{tabular}
\end{table}

\subsection{Geometry Comparison}
\label{geometry}

One of our major advantage is to avoid the artifacts introduced by directly guiding the 3D representation with gradients produced by reward models. In this section, we present the geometry comparison between DreamReward \citep{dreamreward} and our proposed PSD using HPSv2.1 as reward model.

Through Tab. \ref{tab_geometry} and visual examples in Fig. \ref{fig_geometry} and \ref{geometry2}, it's easy to conclude that our results have better geometry, which supports our motivation and claims.   

\begin{table}[h]
\centering
\small
\caption{Geometry assessment. Higher values are better ($\uparrow$). Better results are in bold. \label{tab_geometry}}
\begin{adjustbox}{max width=\textwidth}
\begin{tabular}{lccc}
\toprule
\textbf{Algorithm} & \textbf{Geometric Consistency $\uparrow$} & \textbf{Semantic Consistency $\uparrow$} & \textbf{Structural Consistency $\uparrow$}\\
\midrule
DreamReward & 70.39 & 70.74 & 82.83  \\
Ours & \textbf{80.97} & \textbf{75.83 }& \textbf{84.72} \\
\bottomrule
\end{tabular}
\end{adjustbox}

\end{table}

\begin{figure}[!h] 
\centering 
\includegraphics[width=0.8\textwidth]{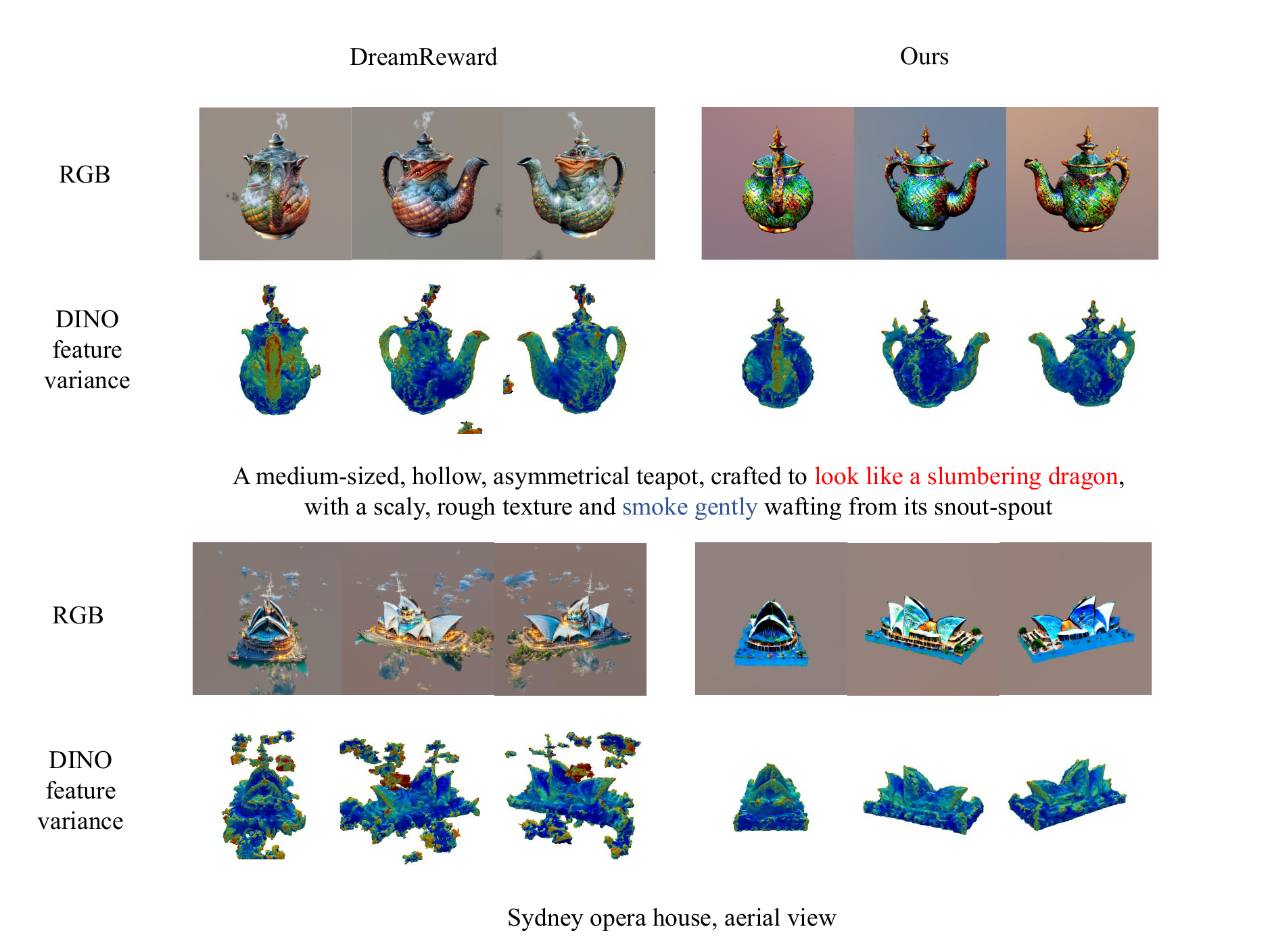} 
\caption{Visual examples of semantic consistency. DINO feature variance is clipped with threshold 0.20 (red).} 
\label{fig_geometry} 
\end{figure}

\begin{figure}[!h] 
\centering 
\includegraphics[width=0.8\textwidth]{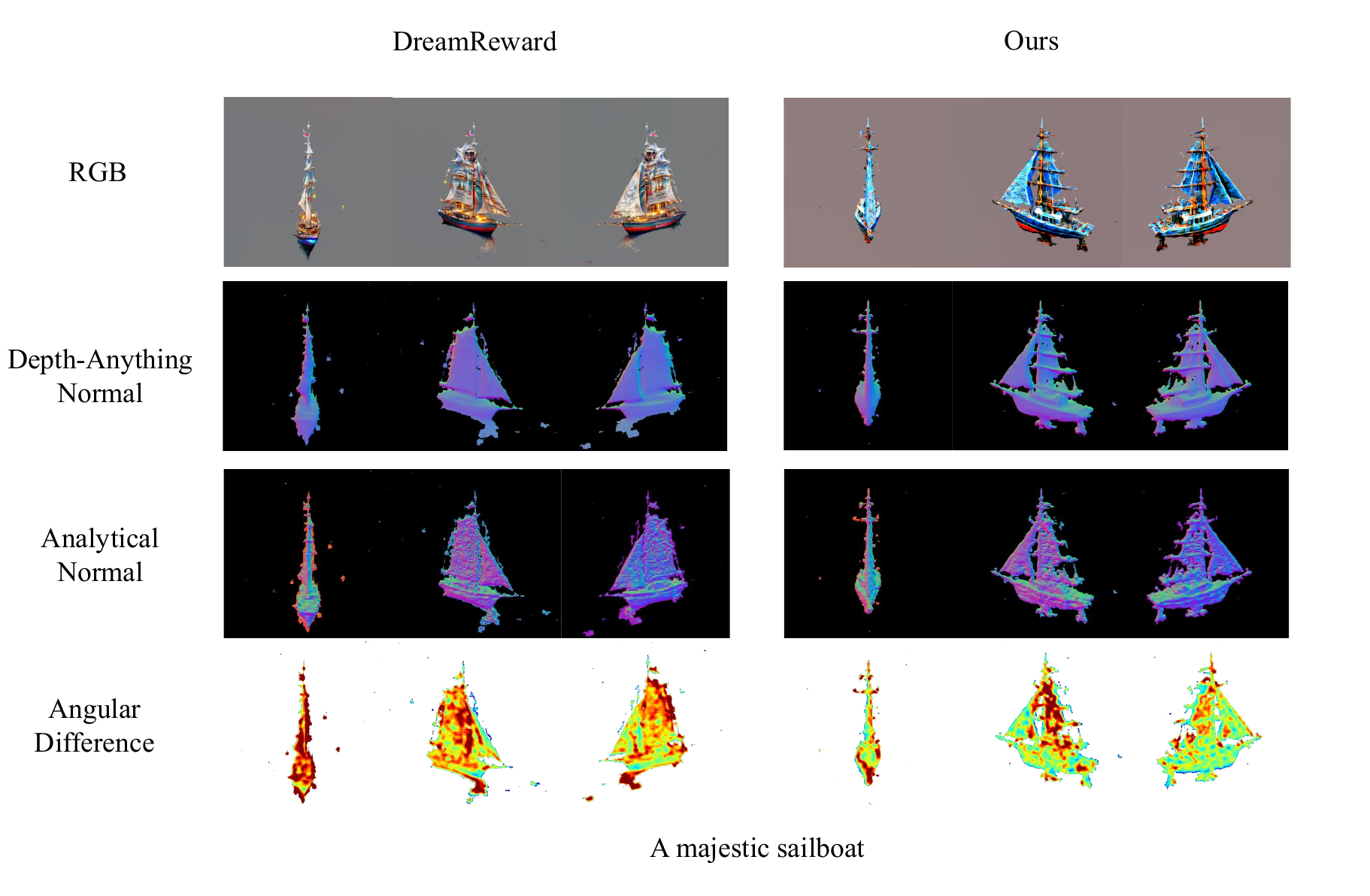} 
\caption{Visual example of geometric consistency. Angular difference is clipped with threshold $23^{\circ}$ (red).} 
\label{geometry2} 
\end{figure}

\subsection{More Qualitative Comparisons}

\begin{table}[h]
    \centering
    \small
    \begin{tabular}{p{6cm}p{6cm}}
 \toprule
\textbf{Column 1} & \textbf{Column 2} \\  
\midrule
1. A DSLR photo of a plate of fried chicken and waffles with maple syrup on them & 21. A tiger playing the violin \\  
2. A beautiful dress made out of garbage bags, on a mannequin. Studio lighting, high quality, high resolution & 22. A wide angle DSLR photo of a colorful rooster \\  
3. A zoomed out DSLR photo of an astronaut chopping vegetables in a sunlit kitchen & 23. A lone, ancient tree stands tall in the middle of a quiet field \\  
4. A wide angle zoomed out DSLR photo of A red dragon dressed in a tuxedo and playing chess. The chess pieces are fashioned after robots & 24. A squirrel dressed like Henry VIII king of England \\  
5. A zoomed out DSLR photo of a pita bread full of hummus and falafel and vegetables & 25. A zoomed out DSLR photo of A punk rock squirrel in a studded leather jacket shouting into a microphone while standing on a stump and holding a beer \\  
6. Several large, solid, cube-shaped parcels, wrapped in brown paper and tied with string, each labeled with a different destination, awaiting dispatch in a post office & 26. A dragon-cat hybrid \\  
7. A large, multi-layered, symmetrical wedding cake, with smooth fondant, delicate piping, and lifelike sugar flowers in full bloom, displayed on a silver stand & 27. A large, hollow, asymmetrically shaped amphitheater, with jagged stone seating, nestled in a natural landscape, a classical play being performed as the sun sets \\  
8. Jellyfish with bioluminescent tentacles shaped like lightning bolts & 28. A compact, cylindrical, vintage pepper mill, with a polished, ornate brass body, slightly worn from use, placed beside a porcelain plate on a checkered tablecloth \\  
9. A Panther De Ville car & 29. A zoomed out DSLR photo of a badger wearing a party hat and blowing out birthday candles on a cake \\  
10. A beagle in a detective's outfit & 30. A mug filled with steaming coffee \\  
11. A wide angle DSLR photo of a squirrel in samurai armor wielding a katana & 31. A zoomed out DSLR photo of a pair of floating chopsticks picking up noodles out of a bowl of ramen \\  
12. A zoomed out DSLR photo of a kangaroo sitting on a bench playing the accordion & 32. A wide angle zoomed out DSLR photo of a skiing penguin wearing a puffy jacket \\  
13. A pair of hiking boots caked with mud at the doorstep of a cabin & 33. A zoomed out DSLR photo of a fox working on a jigsaw puzzle \\  
14. A zoomed out DSLR photo of cats wearing eyeglasses & 34. A zoomed out DSLR photo of a beagle eating a donut \\  
15. A wide angle zoomed out DSLR photo of zoomed out view of Tower Bridge made out of gingerbread and candy & 35. A red panda \\  
16. A beautiful dress made out of fruit, on a mannequin. Studio lighting, high quality, high resolution & 36. A zoomed out DSLR photo of a kingfisher bird \\  
17. A zoomed out DSLR photo of a bear playing electric bass & 37. Clownfish peeking out from sea anemone tendrils \\  
18. A lamp casting shadows on an old, forgotten map & 38. A zoomed out DSLR photo of a rainforest bird mating ritual dance \\  
19. An erupting volcano, aerial view & 39. A chimpanzee with a big grin \\  
20. A tiger karate master & 40. A group of vibrant, chattering parrots perched together \\
\bottomrule
\end{tabular}
    \caption{40 prompts for evaluation in 2-stage NeRF and 3-stage DMTet.}
    \label{tab:prompt}
\end{table}

\begin{figure}[h] 
\centering 
\includegraphics[width=0.6\textwidth]{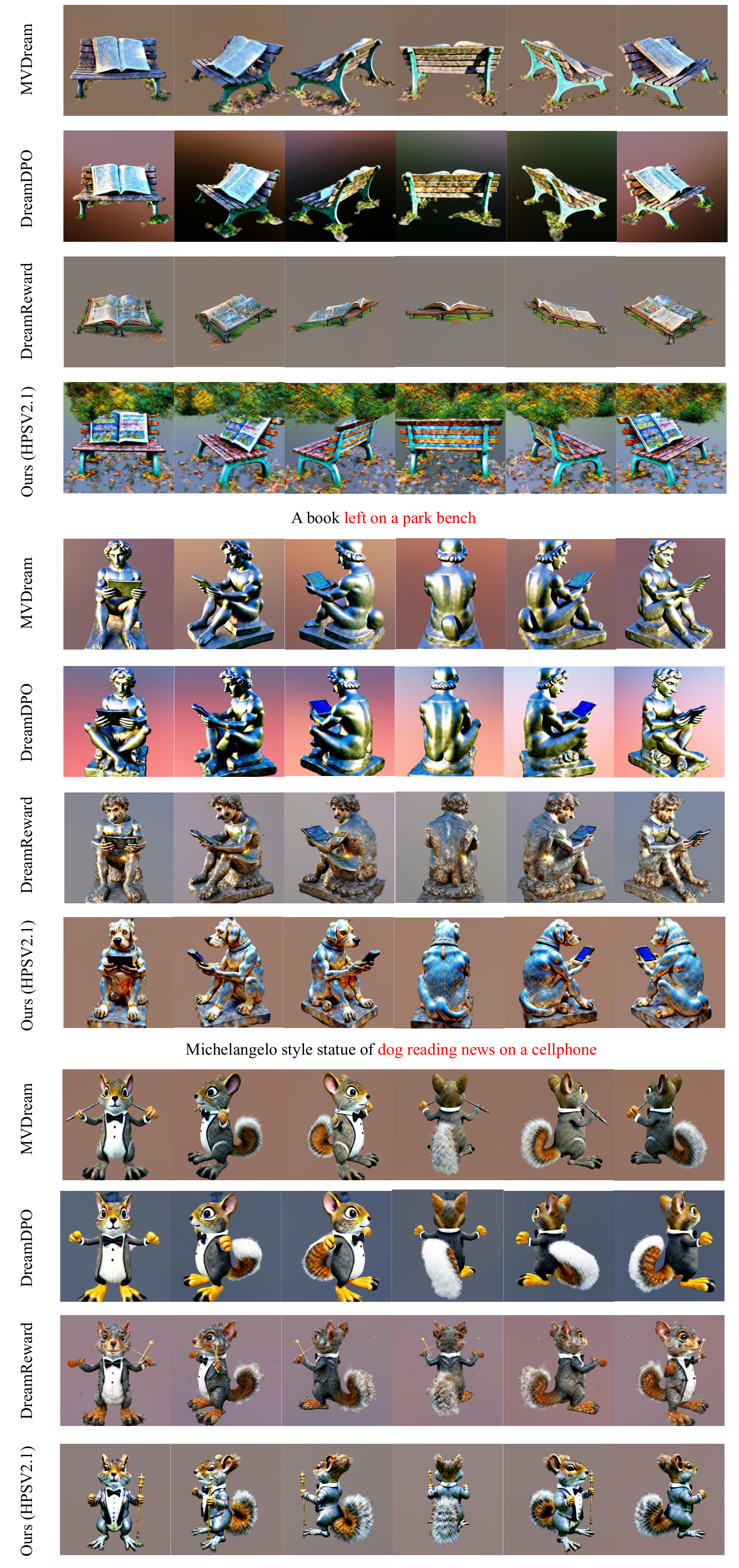} 
\caption{More results of single-stage distillation of MVDream} 
\label{mvdream_appendix1} 
\end{figure}

\begin{figure}[h] 
\centering 
\includegraphics[width=0.65\textwidth]{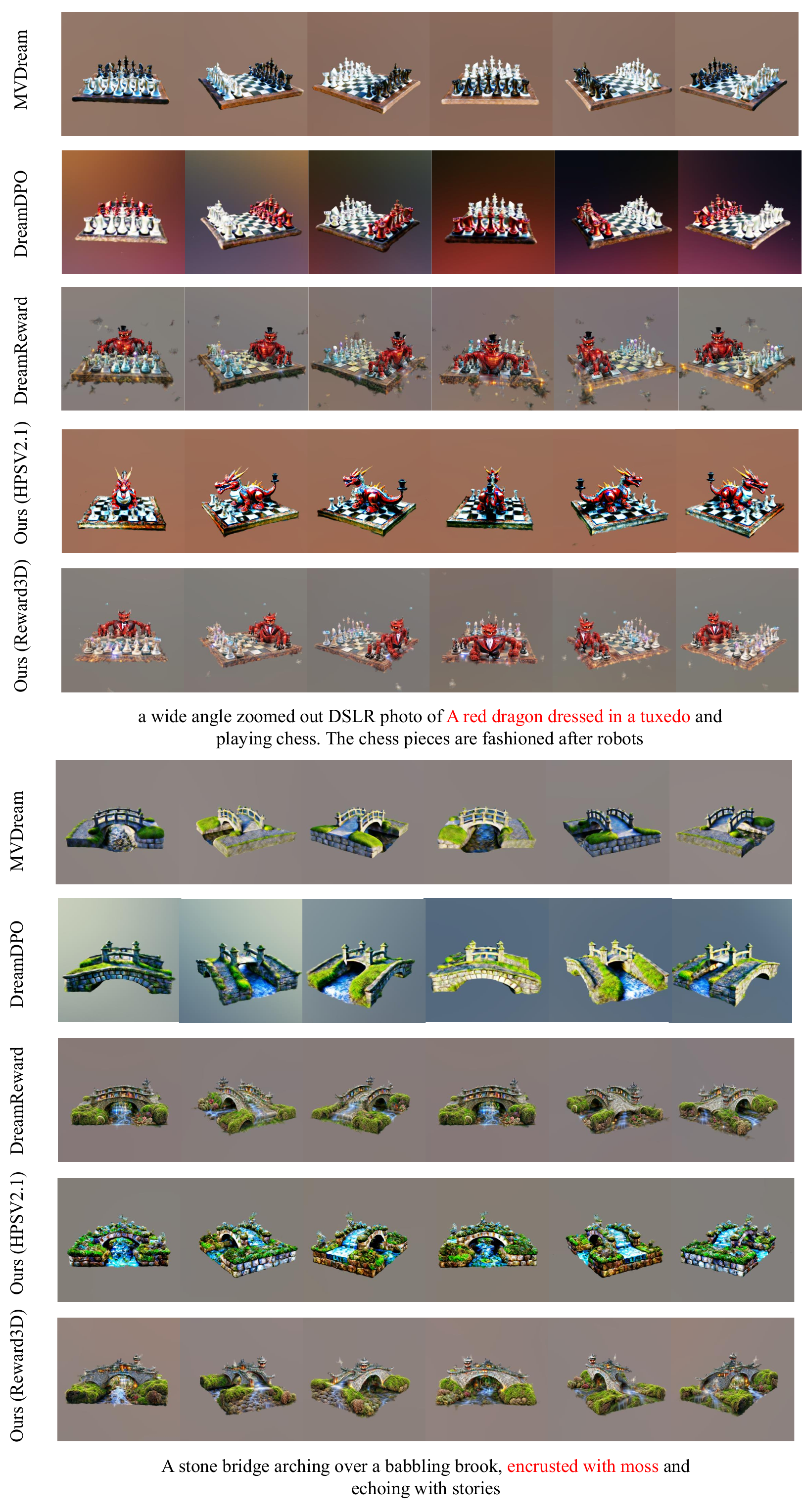} 
\caption{More results of single-stage distillation of MVDream} 
\label{mvdream_appendix2} 
\end{figure}

\begin{figure}[h] 
\centering 
\includegraphics[width=0.85\textwidth]{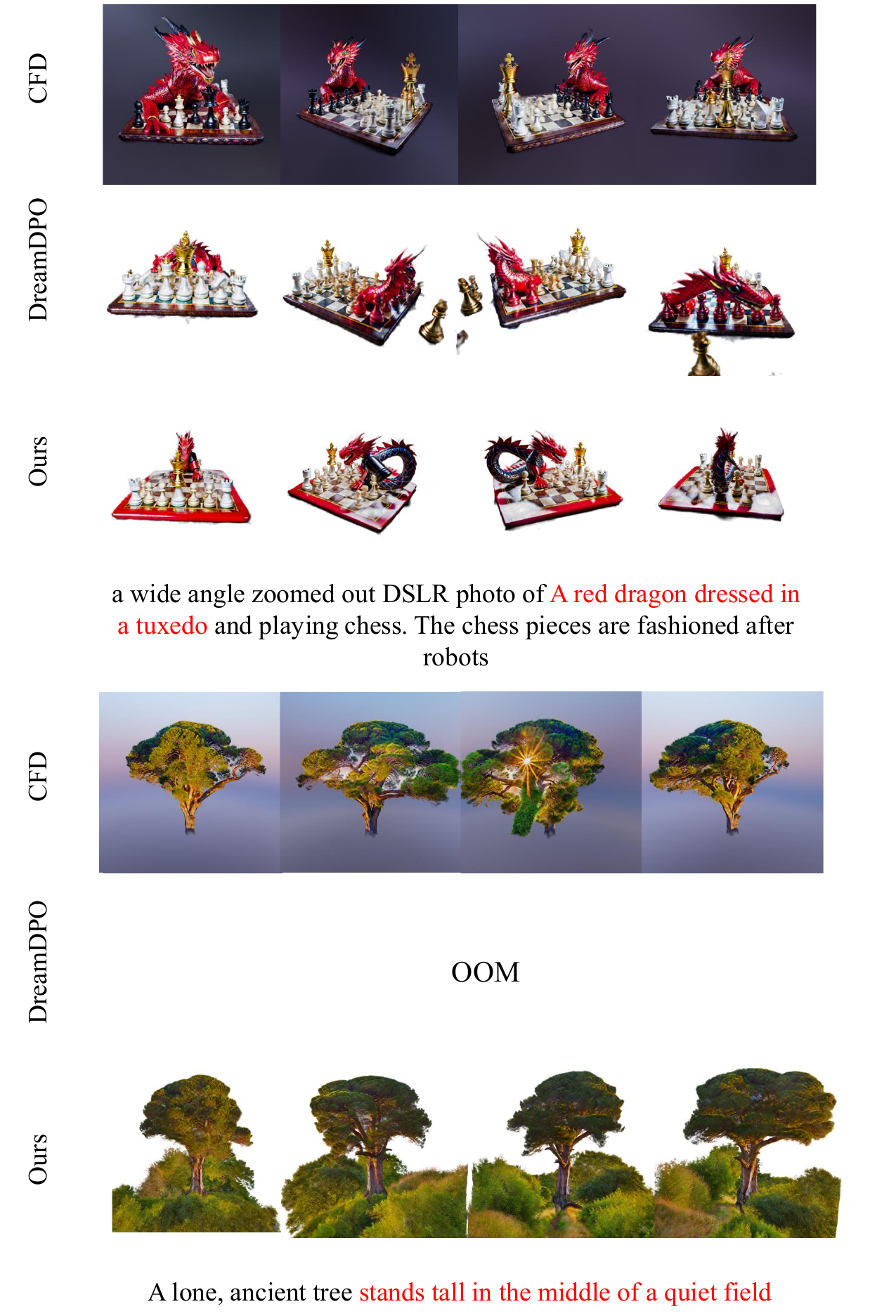} 
\caption{More results of 2-stage NeRF generation.} 
\label{2stage_appendix} 
\end{figure}

\begin{figure}[h] 
\centering 
\includegraphics[width=0.85\textwidth]{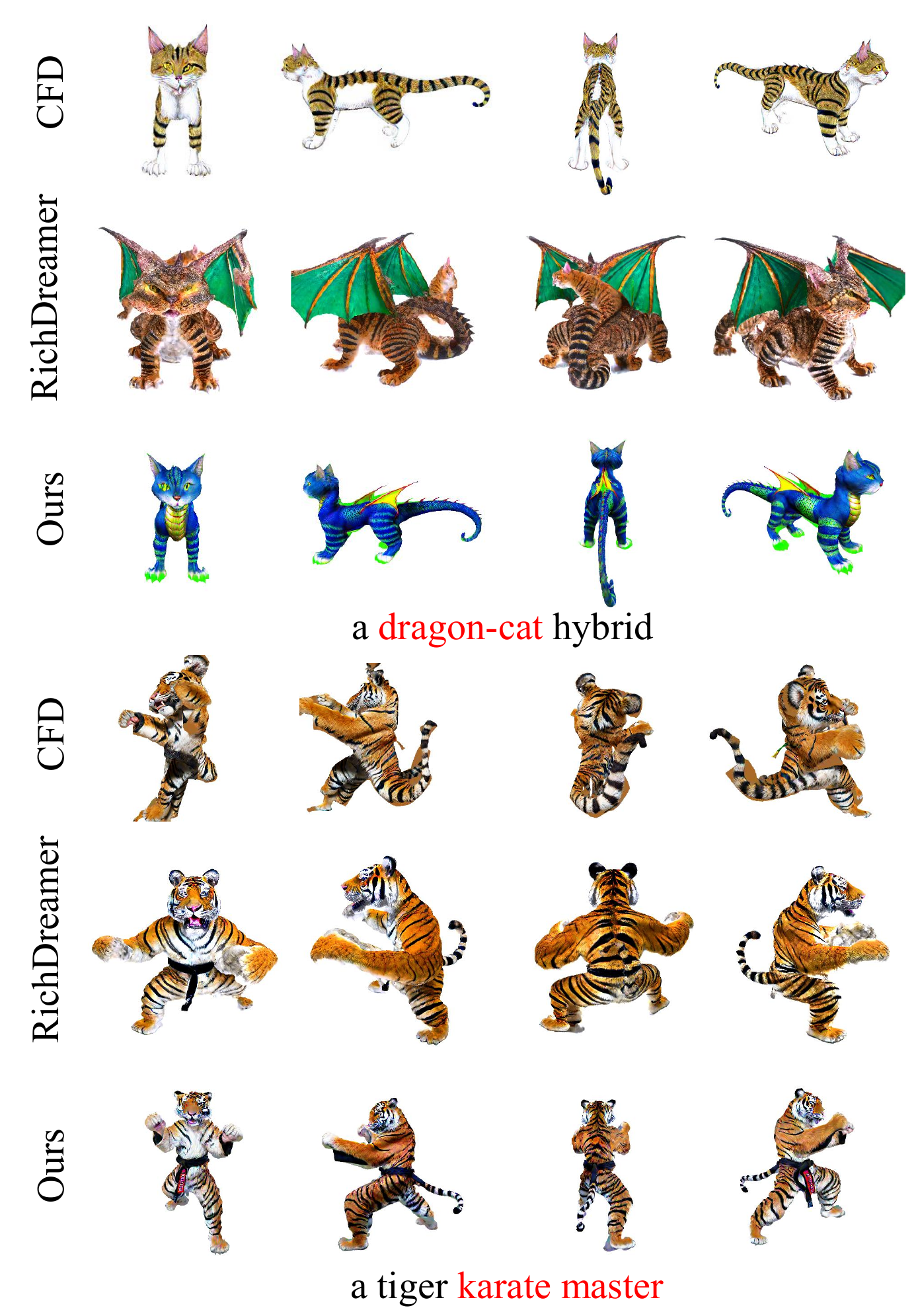} 
\caption{More results of 3-stage DMTet generation.} 
\label{3stage_appendix} 
\end{figure}

\begin{figure}[h] 
\centering 
\includegraphics[width=0.7\textwidth]{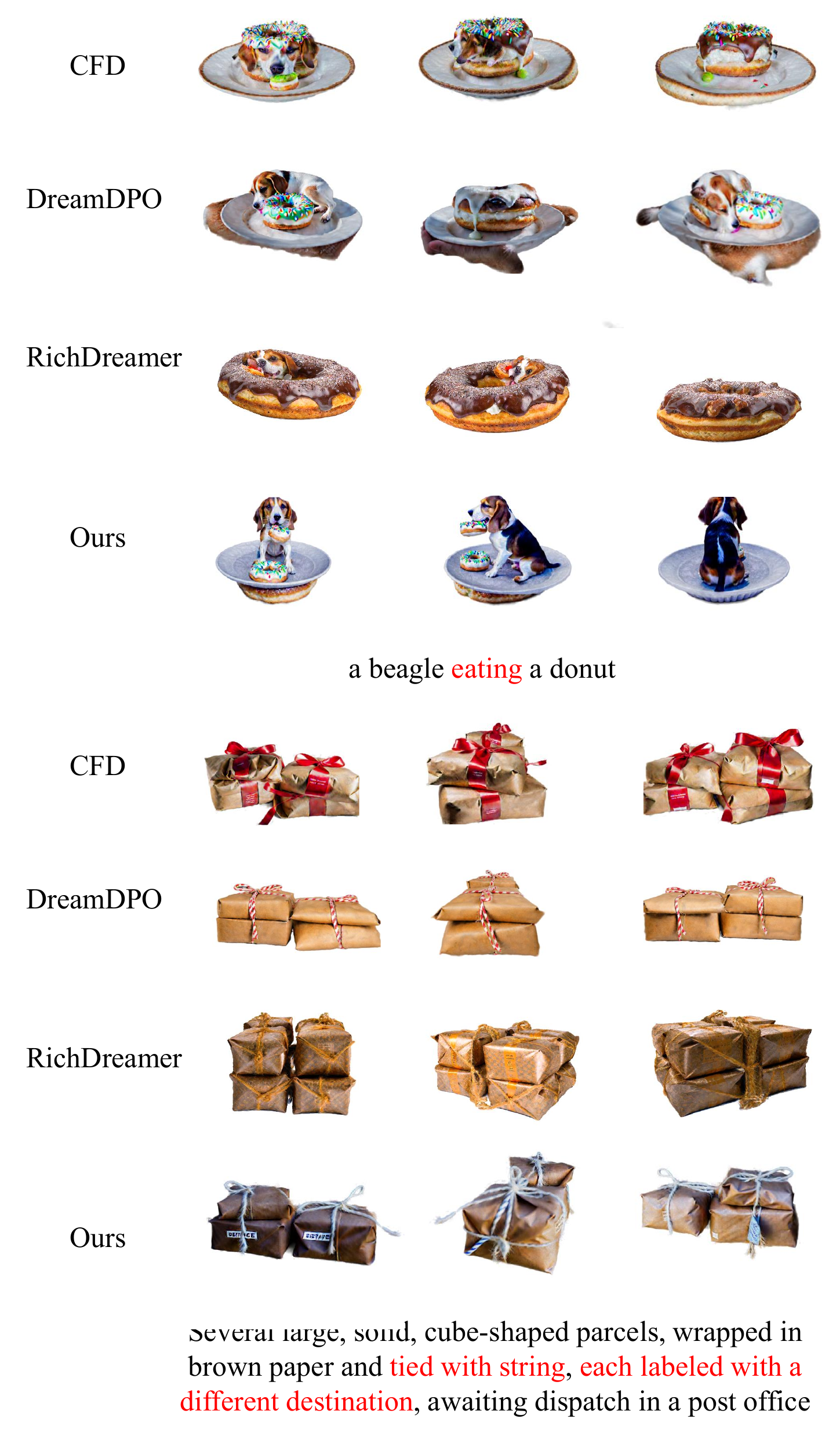} 
\caption{Extended comparison of 2-stage NeRF generation.} 
\label{2stage_extend} 
\end{figure}

%% file: main.bib
@inproceedings{diffusion1,
 author = {Ho, Jonathan and Jain, Ajay and Abbeel, Pieter},
 booktitle = {Advances in Neural Information Processing Systems},
 pages = {6840--6851},
 publisher = {Curran Associates, Inc.},
 title = {Denoising Diffusion Probabilistic Models},
 volume = {33},
 year = {2020}
}

@inproceedings{score,
  title={Score-Based Generative Modeling through Stochastic Differential Equations},
  author={Song, Yang and Sohl-Dickstein, Jascha and Kingma, Diederik P and Kumar, Abhishek and Ermon, Stefano and Poole, Ben},
  booktitle={International Conference on Learning Representations},
year={2021}
}

@inproceedings{diffusion3,
  title={Denoising Diffusion Implicit Models},
  author={Song, Jiaming and Meng, Chenlin and Ermon, Stefano},
  booktitle={International Conference on Learning Representations},
year={2021}
}

@inproceedings{sds,
  title={DreamFusion: Text-to-3D using 2D Diffusion},
  author={Poole, Ben and Jain, Ajay and Barron, Jonathan T and Mildenhall, Ben},
  booktitle={International Conference on Learning Representations},
year={2023}
}

@inproceedings{image1,
 author = {Dhariwal, Prafulla and Nichol, Alexander},
 booktitle = {Advances in Neural Information Processing Systems},
 pages = {8780--8794},
 publisher = {Curran Associates, Inc.},
 title = {Diffusion Models Beat GANs on Image Synthesis},
 volume = {34},
 year = {2021}
}

@InProceedings{sd,
    author    = {Rombach, Robin and Blattmann, Andreas and Lorenz, Dominik and Esser, Patrick and Ommer, Bj\"orn},
    title     = {High-Resolution Image Synthesis With Latent Diffusion Models},
    booktitle = {Proceedings of the IEEE/CVF Conference on Computer Vision and Pattern Recognition (CVPR)},
    month     = {June},
    year      = {2022},
    pages     = {10684-10695}
}

@inproceedings{vsd,
 author = {Wang, Zhengyi and Lu, Cheng and Wang, Yikai and Bao, Fan and LI, Chongxuan and Su, Hang and Zhu, Jun},
 booktitle = {Advances in Neural Information Processing Systems},
 pages = {8406--8441},
 publisher = {Curran Associates, Inc.},
 title = {ProlificDreamer: High-Fidelity and Diverse Text-to-3D Generation with Variational Score Distillation},
 volume = {36},
 year = {2023}
}

@inproceedings{sdi,
 author = {Lukoianov, Artem and de Oc\'{a}riz Borde, Haitz S\'{a}ez and Greenewald, Kristjan and Guizilini, Vitor Campagnolo and Bagautdinov, Timur and Sitzmann, Vincent and Solomon, Justin},
 booktitle = {Advances in Neural Information Processing Systems},
 pages = {26011--26044},
 publisher = {Curran Associates, Inc.},
 title = {Score Distillation via Reparametrized DDIM},
 volume = {37},
 year = {2024}
}

@inproceedings{bridge,
 author = {McAllister, David and Ge, Songwei and Huang, Jia-Bin and Jacobs, David W. and Efros, Alexei A. and Holynski, Aleksander and Kanazawa, Angjoo},
 booktitle = {Advances in Neural Information Processing Systems},
 pages = {33779--33804},
 publisher = {Curran Associates, Inc.},
 title = {Rethinking Score Distillation as a Bridge Between Image Distributions},
 volume = {37},
 year = {2024}
}

@InProceedings{connect3d,
author="Li, Zongrui
and Hu, Minghui
and Zheng, Qian
and Jiang, Xudong",
title="Connecting Consistency Distillation to Score Distillation for Text-to-3D Generation",
booktitle="Computer Vision -- ECCV 2024",
year="2025",
publisher="Springer Nature Switzerland",
address="Cham",
pages="274--291",
isbn="978-3-031-72775-7"
}

@InProceedings{jointdreamer,
author="Jiang, Chenhan
and Zeng, Yihan
and Hu, Tianyang
and Xu, Songcun
and Zhang, Wei
and Xu, Hang
and Yeung, Dit-Yan",
title="JointDreamer: Ensuring Geometry Consistency and Text Congruence in Text-to-3D Generation via Joint Score Distillation",
booktitle="Computer Vision -- ECCV 2024",
year="2025",
publisher="Springer Nature Switzerland",
address="Cham",
pages="439--456",
isbn="978-3-031-73347-5"
}

@InProceedings{DreamMesh,
author="Yang, Haibo
and Chen, Yang
and Pan, Yingwei
and Yao, Ting
and Chen, Zhineng
and Wu, Zuxuan
and Jiang, Yu-Gang
and Mei, Tao",
title="DreamMesh: Jointly Manipulating and Texturing Triangle Meshes for Text-to-3D Generation",
booktitle="Computer Vision -- ECCV 2024",
year="2025",
publisher="Springer Nature Switzerland",
address="Cham",
pages="162--178",
isbn="978-3-031-73202-7"
}

@InProceedings{dmd,
    author    = {Yin, Tianwei and Gharbi, Micha\"el and Zhang, Richard and Shechtman, Eli and Durand, Fr\'edo and Freeman, William T. and Park, Taesung},
    title     = {One-step Diffusion with Distribution Matching Distillation},
    booktitle = {Proceedings of the IEEE/CVF Conference on Computer Vision and Pattern Recognition (CVPR)},
    month     = {June},
    year      = {2024},
    pages     = {6613-6623}
}

@inproceedings{dmd2,
 author = {Yin, Tianwei and Gharbi, Micha\"{e}l and Park, Taesung and Zhang, Richard and Shechtman, Eli and Durand, Fr\'{e}do and Freeman, William T.},
 booktitle = {Advances in Neural Information Processing Systems},
 pages = {47455--47487},
 publisher = {Curran Associates, Inc.},
 title = {Improved Distribution Matching Distillation for Fast Image Synthesis},
 volume = {37},
 year = {2024}
}

@InProceedings{dmd_viedo,
    author    = {Yin, Tianwei and Zhang, Qiang and Zhang, Richard and Freeman, William T. and Durand, Fredo and Shechtman, Eli and Huang, Xun},
    title     = {From Slow Bidirectional to Fast Autoregressive Video Diffusion Models},
    booktitle = {Proceedings of the IEEE/CVF Conference on Computer Vision and Pattern Recognition (CVPR)},
    month     = {June},
    year      = {2025},
    pages     = {22963-22974}
}

@InProceedings{character,
    author    = {Li, Xuan and Ma, Qianli and Lin, Tsung-Yi and Chen, Yongxin and Jiang, Chenfanfu and Liu, Ming-Yu and Xiang, Donglai},
    title     = {Articulated Kinematics Distillation from Video Diffusion Models},
    booktitle = {Proceedings of the IEEE/CVF Conference on Computer Vision and Pattern Recognition (CVPR)},
    month     = {June},
    year      = {2025},
    pages     = {17571-17581}
}

@InProceedings{depth,
    author    = {Pham, Duc-Hai and Do, Tung and Nguyen, Phong and Hua, Binh-Son and Nguyen, Khoi and Nguyen, Rang},
    title     = {SharpDepth: Sharpening Metric Depth Predictions Using Diffusion Distillation},
    booktitle = {Proceedings of the IEEE/CVF Conference on Computer Vision and Pattern Recognition (CVPR)},
    month     = {June},
    year      = {2025},
    pages     = {17060-17069}
}

@InProceedings{dreamreward,
author="Ye, JunLiang
and Liu, Fangfu
and Li, Qixiu
and Wang, Zhengyi
and Wang, Yikai
and Wang, Xinzhou
and Duan, Yueqi
and Zhu, Jun",
title="DreamReward: Text-to-3D Generation with Human Preference",
booktitle="Computer Vision -- ECCV 2024",
year="2025",
publisher="Springer Nature Switzerland",
address="Cham",
pages="259--276",
isbn="978-3-031-72897-6"
}

@inproceedings{dreamalign,
  title={DreamAlign: Dynamic text-to-3D optimization with human preference alignment},
  author={Liu, Gaofeng and Ma, Zhiyuan and Fang, Tao},
  booktitle={Proceedings of the AAAI Conference on Artificial Intelligence},
  volume={39},
  number={5},
  pages={5424--5432},
  year={2025}
}

@article{dreamdpo,
  title={Dreamdpo: Aligning text-to-3d generation with human preferences via direct preference optimization},
  author={Zhou, Zhenglin and Xia, Xiaobo and Ma, Fan and Fan, Hehe and Yang, Yi and Chua, Tat-Seng},
  journal={arXiv preprint arXiv:2502.04370},
  year={2025}
}

@ARTICLE{dreamrewardx,
  author={Liu, Fangfu and Ye, Junliang and Wang, Yikai and Wang, Hanyang and Wang, Zhengyi and Zhu, Jun and Duan, Yueqi},
  journal={IEEE Transactions on Pattern Analysis and Machine Intelligence}, 
  title={DreamReward-X: Boosting High-Quality 3D Generation with Human Preference Alignment}, 
  year={2025},
  volume={},
  number={},
  pages={1-14},

  doi={10.1109/TPAMI.2025.3609680}}

@article{dreamcs,
  title={DreamCS: Geometry-Aware Text-to-3D Generation with Unpaired 3D Reward Supervision},
  author={Zou, Xiandong and Xia, Ruihao and Wang, Hongsong and Zhou, Pan},
  journal={arXiv preprint arXiv:2506.09814},
  year={2025}
}

@article{cfg,
  title={Classifier-free diffusion guidance},
  author={Ho, Jonathan and Salimans, Tim},
  journal={arXiv preprint arXiv:2207.12598},
  year={2022}
}

@InProceedings{neg1,
author="Ban, Yuanhao
and Wang, Ruochen
and Zhou, Tianyi
and Cheng, Minhao
and Gong, Boqing
and Hsieh, Cho-Jui",
title="Understanding the Impact of Negative Prompts: When and How Do They Take Effect?",
booktitle="Computer Vision -- ECCV 2024",
year="2025",
publisher="Springer Nature Switzerland",
address="Cham",
pages="190--206",
isbn="978-3-031-73024-5"
}

@InProceedings{neg2,
  title = 	 {On Discrete Prompt Optimization for Diffusion Models},
  author =       {Wang, Ruochen and Liu, Ting and Hsieh, Cho-Jui and Gong, Boqing},
  booktitle = 	 {Proceedings of the 41st International Conference on Machine Learning},
  pages = 	 {50992--51011},
  year = 	 {2024},
  volume = 	 {235},
  series = 	 {Proceedings of Machine Learning Research},
  month = 	 {21--27 Jul},
  publisher =    {PMLR},
}

@InProceedings{reneg,
    author    = {Li, Xiaomin and Liu, Yixuan and Isobe, Takashi and Jia, Xu and Cui, Qinpeng and Zhou, Dong and Li, Dong and He, You and Lu, Huchuan and Wang, Zhongdao and Barsoum, Emad},
    title     = {ReNeg: Learning Negative Embedding with Reward Guidance},
    booktitle = {Proceedings of the IEEE/CVF Conference on Computer Vision and Pattern Recognition (CVPR)},
    month     = {June},
    year      = {2025},
    pages     = {23636-23645}
}

@inproceedings{edm,
 author = {Karras, Tero and Aittala, Miika and Aila, Timo and Laine, Samuli},
 booktitle = {Advances in Neural Information Processing Systems},
 pages = {26565--26577},
 publisher = {Curran Associates, Inc.},
 title = {Elucidating the Design Space of Diffusion-Based Generative Models},
 volume = {35},
 year = {2022}
}

@article{fsd,
  title={Flow Score Distillation for Diverse Text-to-3D Generation},
  author={Yan, Runjie and Wu, Kailu and Ma, Kaisheng},
  journal={arXiv preprint arXiv:2405.10988},
  year={2024}
}

@inproceedings{cfd,
  title={Consistent Flow Distillation for Text-to-3D Generation},
  author={Yan, Runjie and Chen, Yinbo and Wang, Xiaolong},
  booktitle={International Conference on Learning Representations},
year={2025}
}

@inproceedings{rlhf,
 author = {Christiano, Paul F and Leike, Jan and Brown, Tom and Martic, Miljan and Legg, Shane and Amodei, Dario},
 booktitle = {Advances in Neural Information Processing Systems},
 pages = {},
 publisher = {Curran Associates, Inc.},
 title = {Deep Reinforcement Learning from Human Preferences},
 volume = {30},
 year = {2017}
}

@inproceedings{test,
  title={Test-time Alignment of Diffusion Models without Reward Over-optimization},
  author={Kim, Sunwoo and Kim, Minkyu and Park, Dongmin},
  booktitle={International Conference on Learning Representations},
year = {2025}
}

@inproceedings{dpo,
 author = {Rafailov, Rafael and Sharma, Archit and Mitchell, Eric and Manning, Christopher D and Ermon, Stefano and Finn, Chelsea},
 booktitle = {Advances in Neural Information Processing Systems},
 pages = {53728--53741},
 publisher = {Curran Associates, Inc.},
 title = {Direct Preference Optimization: Your Language Model is Secretly a Reward Model},
 volume = {36},
 year = {2023}
}

@article{bt,
  title={Rank analysis of incomplete block designs: I. the method of paired comparisons},
  author={Bradley, Ralph Allan and Terry, Milton E},
  journal={Biometrika},
  volume={39},
  number={3/4},
  pages={324--345},
  year={1952},
  publisher={JSTOR}
}

@article{tweedie,
  title={Tweedie’s formula and selection bias},
  author={Efron, Bradley},
  journal={Journal of the American Statistical Association},
  volume={106},
  number={496},
  pages={1602--1614},
  year={2011},
  publisher={Taylor \& Francis}
}

@inproceedings{csd,
  title={Text-to-3D with Classifier Score Distillation},
  author={Yu, Xin and Guo, Yuan-Chen and Li, Yangguang and Liang, Ding and Zhang, Song-Hai and QI, XIAOJUAN},
  booktitle={International Conference on Learning Representations},
year={2024},
}

@article{segment,
  title={SegmentDreamer: Towards High-fidelity Text-to-3D Synthesis with Segmented Consistency Trajectory Distillation},
  author={Zhu, Jiahao and Chen, Zixuan and Wang, Guangcong and Xie, Xiaohua and Zhou, Yi},
  journal={arXiv preprint arXiv:2507.05256},
  year={2025}
}

@article{nerf,
author = {Mildenhall, Ben and Srinivasan, Pratul P. and Tancik, Matthew and Barron, Jonathan T. and Ramamoorthi, Ravi and Ng, Ren},
title = {NeRF: representing scenes as neural radiance fields for view synthesis},
year = {2021},
issue_date = {January 2022},
publisher = {Association for Computing Machinery},
address = {New York, NY, USA},
volume = {65},
number = {1},
issn = {0001-0782},
doi = {10.1145/3503250},
journal = {Commun. ACM},
month = dec,
pages = {99–106},
numpages = {8}
}

@inproceedings{dmtet,
 author = {Shen, Tianchang and Gao, Jun and Yin, Kangxue and Liu, Ming-Yu and Fidler, Sanja},
 booktitle = {Advances in Neural Information Processing Systems},
 pages = {6087--6101},
 publisher = {Curran Associates, Inc.},
 title = {Deep Marching Tetrahedra: a Hybrid Representation for High-Resolution 3D Shape Synthesis},
 volume = {34},
 year = {2021}
}

@inproceedings{lora,
  title={LoRA: Low-Rank Adaptation of Large Language Models},
  author={Hu, Edward J and Wallis, Phillip and Allen-Zhu, Zeyuan and Li, Yuanzhi and Wang, Shean and Wang, Lu and Chen, Weizhu and others},
  booktitle={International Conference on Learning Representations},
year={2022},
}

@inproceedings{mvdream,
  title={MVDream: Multi-view Diffusion for 3D Generation},
  author={Shi, Yichun and Wang, Peng and Ye, Jianglong and Mai, Long and Li, Kejie and Yang, Xiao},
  booktitle={International Conference on Learning Representations},
year={2024},
}

@misc{threestudio,
  month = {06},
  title = {threestudio-project/threestudio},
  url = {https://github.com/threestudio-project/threestudio},
  year = {2024},
  organization = {GitHub}
}

@InProceedings{eval3d,
    author    = {Duggal, Shivam and Hu, Yushi and Michel, Oscar and Kembhavi, Aniruddha and Freeman, William T. and Smith, Noah A. and Krishna, Ranjay and Torralba, Antonio and Farhadi, Ali and Ma, Wei-Chiu},
    title     = {Eval3D: Interpretable and Fine-grained Evaluation for 3D Generation},
    booktitle = {Proceedings of the IEEE/CVF Conference on Computer Vision and Pattern Recognition (CVPR)},
    month     = {June},
    year      = {2025},
    pages     = {13326-13336}
}

@inproceedings{imagereward,
 author = {Xu, Jiazheng and Liu, Xiao and Wu, Yuchen and Tong, Yuxuan and Li, Qinkai and Ding, Ming and Tang, Jie and Dong, Yuxiao},
 booktitle = {Advances in Neural Information Processing Systems},
 pages = {15903--15935},
 publisher = {Curran Associates, Inc.},
 title = {ImageReward: Learning and Evaluating Human Preferences for Text-to-Image Generation},
 volume = {36},
 year = {2023}
}

@inproceedings{pickscore,
 author = {Kirstain, Yuval and Polyak, Adam and Singer, Uriel and Matiana, Shahbuland and Penna, Joe and Levy, Omer},
 booktitle = {Advances in Neural Information Processing Systems},
 pages = {36652--36663},
 publisher = {Curran Associates, Inc.},
 title = {Pick-a-Pic: An Open Dataset of User Preferences for Text-to-Image Generation},
 volume = {36},
 year = {2023}
}

@misc{laion,
  author = {Schuhmann, Christoph},
  title = {LAION-Aesthetics | LAION},
  url = {https://laion.ai/blog/laion-aesthetics/},
  organization = {laion.ai}
}

@article{qwen,
  title={Qwen2. 5-vl technical report},
  author={Bai, Shuai and Chen, Keqin and Liu, Xuejing and Wang, Jialin and Ge, Wenbin and Song, Sibo and Dang, Kai and Wang, Peng and Wang, Shijie and Tang, Jun and others},
  journal={arXiv preprint arXiv:2502.13923},
  year={2025}
}

@InProceedings{mps,
    author    = {Zhang, Sixian and Wang, Bohan and Wu, Junqiang and Li, Yan and Gao, Tingting and Zhang, Di and Wang, Zhongyuan},
    title     = {Learning Multi-Dimensional Human Preference for Text-to-Image Generation},
    booktitle = {Proceedings of the IEEE/CVF Conference on Computer Vision and Pattern Recognition (CVPR)},
    month     = {June},
    year      = {2024},
    pages     = {8018-8027}
}

@InProceedings{richdreamer,
    author    = {Qiu, Lingteng and Chen, Guanying and Gu, Xiaodong and Zuo, Qi and Xu, Mutian and Wu, Yushuang and Yuan, Weihao and Dong, Zilong and Bo, Liefeng and Han, Xiaoguang},
    title     = {RichDreamer: A Generalizable Normal-Depth Diffusion Model for Detail Richness in Text-to-3D},
    booktitle = {Proceedings of the IEEE/CVF Conference on Computer Vision and Pattern Recognition (CVPR)},
    month     = {June},
    year      = {2024},
    pages     = {9914-9925}
}

@article{adam,
  title={Adam: A method for stochastic optimization},
  author={Kingma, Diederik P and Ba, Jimmy},
  journal={arXiv preprint arXiv:1412.6980},
  year={2014}
}

@InProceedings{dreamsampler,
author="Kim, Jeongsol
and Park, Geon Yeong
and Ye, Jong Chul",
title="DreamSampler: Unifying Diffusion Sampling and Score Distillation for Image Manipulation",
booktitle="Computer Vision -- ECCV 2024",
year="2025",
publisher="Springer Nature Switzerland",
address="Cham",
pages="398--414",
isbn="978-3-031-73007-8"
}

@article{centric,
  title={Reward-Instruct: A Reward-Centric Approach to Fast Photo-Realistic Image Generation},
  author={Luo, Yihong and Hu, Tianyang and Luo, Weijian and Kawaguchi, Kenji and Tang, Jing},
  journal={arXiv preprint arXiv:2503.13070},
  year={2025}
}

@article{parti,
  title={Scaling autoregressive models for content-rich text-to-image generation},
  author={Yu, Jiahui and Xu, Yuanzhong and Koh, Jing Yu and Luong, Thang and Baid, Gunjan and Wang, Zirui and Vasudevan, Vijay and Ku, Alexander and Yang, Yinfei and Ayan, Burcu Karagol and others},
  journal={arXiv preprint arXiv:2206.10789},
  year={2022}
}

@InProceedings{trellis,
    author    = {Xiang, Jianfeng and Lv, Zelong and Xu, Sicheng and Deng, Yu and Wang, Ruicheng and Zhang, Bowen and Chen, Dong and Tong, Xin and Yang, Jiaolong},
    title     = {Structured 3D Latents for Scalable and Versatile 3D Generation},
    booktitle = {Proceedings of the IEEE/CVF Conference on Computer Vision and Pattern Recognition (CVPR)},
    month     = {June},
    year      = {2025},
    pages     = {21469-21480}
}

@InProceedings{sd3,
  title = 	 {Scaling Rectified Flow Transformers for High-Resolution Image Synthesis},
  author =       {Esser, Patrick and Kulal, Sumith and Blattmann, Andreas and Entezari, Rahim and M\"{u}ller, Jonas and Saini, Harry and Levi, Yam and Lorenz, Dominik and Sauer, Axel and Boesel, Frederic and Podell, Dustin and Dockhorn, Tim and English, Zion and Rombach, Robin},
  booktitle = 	 {Proceedings of the 41st International Conference on Machine Learning},
  pages = 	 {12606--12633},
  year = 	 {2024},
  volume = 	 {235},
  series = 	 {Proceedings of Machine Learning Research},
  month = 	 {21--27 Jul},
  publisher =    {PMLR},
}

@inproceedings{sdt,
author = {Sauer, Axel and Boesel, Frederic and Dockhorn, Tim and Blattmann, Andreas and Esser, Patrick and Rombach, Robin},
title = {Fast High-Resolution Image Synthesis with Latent Adversarial Diffusion Distillation},
year = {2024},
isbn = {9798400711312},
publisher = {Association for Computing Machinery},
address = {New York, NY, USA},
doi = {10.1145/3680528.3687625},
booktitle = {SIGGRAPH Asia 2024 Conference Papers},
articleno = {106},
numpages = {11},
location = {Tokyo, Japan},
series = {SA '24}
}

@misc{svd,
      title={Stable Video Diffusion: Scaling Latent Video Diffusion Models to Large Datasets}, 
      author={Andreas Blattmann and Tim Dockhorn and Sumith Kulal and Daniel Mendelevitch and Maciej Kilian and Dominik Lorenz and Yam Levi and Zion English and Vikram Voleti and Adam Letts and Varun Jampani and Robin Rombach},
      year={2023},
      journal={arXiv preprint arXiv:2311.15127},
}

@misc{hps,
      title={Human Preference Score v2: A Solid Benchmark for Evaluating Human Preferences of Text-to-Image Synthesis}, 
      author={Xiaoshi Wu and Yiming Hao and Keqiang Sun and Yixiong Chen and Feng Zhu and Rui Zhao and Hongsheng Li},
      year={2023},
      journal={arXiv preprint arXiv:2306.09341},
}

@inproceedings{dpok,
 author = {Fan, Ying and Watkins, Olivia and Du, Yuqing and Liu, Hao and Ryu, Moonkyung and Boutilier, Craig and Abbeel, Pieter and Ghavamzadeh, Mohammad and Lee, Kangwook and Lee, Kimin},
 booktitle = {Advances in Neural Information Processing Systems},
 pages = {79858--79885},
 publisher = {Curran Associates, Inc.},
 title = {DPOK: Reinforcement Learning for Fine-tuning Text-to-Image Diffusion Models},
 volume = {36},
 year = {2023}
}

@inproceedings{ddpo,
  title={Training Diffusion Models with Reinforcement Learning},
  author={Black, Kevin and Janner, Michael and Du, Yilun and Kostrikov, Ilya and Levine, Sergey},
  booktitle={International Conference on Learning Representations},
year = {2024}
}

@inproceedings{hacking,
  title={Scaling laws for reward model overoptimization},
  author={Gao, Leo and Schulman, John and Hilton, Jacob},
  booktitle={International Conference on Machine Learning},
  pages={10835--10866},
  year={2023},
  organization={PMLR}
}

@InProceedings{diffusiondpo,
    author    = {Wallace, Bram and Dang, Meihua and Rafailov, Rafael and Zhou, Linqi and Lou, Aaron and Purushwalkam, Senthil and Ermon, Stefano and Xiong, Caiming and Joty, Shafiq and Naik, Nikhil},
    title     = {Diffusion Model Alignment Using Direct Preference Optimization},
    booktitle = {Proceedings of the IEEE/CVF Conference on Computer Vision and Pattern Recognition (CVPR)},
    month     = {June},
    year      = {2024},
    pages     = {8228-8238}
}

@InProceedings{d3po,
    author    = {Yang, Kai and Tao, Jian and Lyu, Jiafei and Ge, Chunjiang and Chen, Jiaxin and Shen, Weihan and Zhu, Xiaolong and Li, Xiu},
    title     = {Using Human Feedback to Fine-tune Diffusion Models without Any Reward Model},
    booktitle = {Proceedings of the IEEE/CVF Conference on Computer Vision and Pattern Recognition (CVPR)},
    month     = {June},
    year      = {2024},
    pages     = {8941-8951}
}

@inproceedings{dspo,
  title={DSPO: Direct score preference optimization for diffusion model alignment},
  author={Zhu, Huaisheng and Xiao, Teng and Honavar, Vasant G},
  booktitle={International Conference on Learning Representations},
  year={2025}
}

@InProceedings{consistent3d,
    author    = {Wu, Zike and Zhou, Pan and Yi, Xuanyu and Yuan, Xiaoding and Zhang, Hanwang},
    title     = {Consistent3D: Towards Consistent High-Fidelity Text-to-3D Generation with Deterministic Sampling Prior},
    booktitle = {Proceedings of the IEEE/CVF Conference on Computer Vision and Pattern Recognition (CVPR)},
    month     = {June},
    year      = {2024},
    pages     = {9892-9902}
}

@InProceedings{inpo,
    author    = {Lu, Yunhong and Wang, Qichao and Cao, Hengyuan and Wang, Xierui and Xu, Xiaoyin and Zhang, Min},
    title     = {InPO: Inversion Preference Optimization with Reparametrized DDIM for Efficient Diffusion Model Alignment},
    booktitle = {Proceedings of the IEEE/CVF Conference on Computer Vision and Pattern Recognition (CVPR)},
    month     = {June},
    year      = {2025},
    pages     = {28629-28639}
}

@article{smooth,
  title={Smoothed Preference Optimization via ReNoise Inversion for Aligning Diffusion Models with Varied Human Preferences},
  author={Lu, Yunhong and Wang, Qichao and Cao, Hengyuan and Xu, Xiaoyin and Zhang, Min},
  journal={arXiv preprint arXiv:2506.02698},
  year={2025}
}

@inproceedings{npo,
  title={Diffusion-NPO: Negative Preference Optimization for Better Preference Aligned Generation of Diffusion Models},
  author={Wang, Fu-Yun and Shui, Yunhao and Piao, Jingtan and Sun, Keqiang and Li, Hongsheng},
  booktitle={International Conference on Learning Representations},
year={2025}
}

@article{self,
  title={Self-NPO: Negative Preference Optimization of Diffusion Models by Simply Learning from Itself without Explicit Preference Annotations},
  author={Wang, Fu-Yun and Sun, Keqiang and Teng, Yao and Liu, Xihui and Song, Jiaming and Li, Hongsheng},
  journal={arXiv preprint arXiv:2505.11777},
  year={2025}
}

@InProceedings{doodl,
    author    = {Wallace, Bram and Gokul, Akash and Ermon, Stefano and Naik, Nikhil},
    title     = {End-to-End Diffusion Latent Optimization Improves Classifier Guidance},
    booktitle = {Proceedings of the IEEE/CVF International Conference on Computer Vision (ICCV)},
    month     = {October},
    year      = {2023},
    pages     = {7280-7290}
}

@inproceedings{demon,
  title={Training-Free Diffusion Model Alignment with Sampling Demons},
  author={Yeh, Po-Hung and Lee, Kuang-Huei and Chen, Jun-cheng},
  booktitle={International Conference on Learning Representations},
year      = {2025},
}

@article{dno,
  title={Inference-Time Alignment of Diffusion Models with Direct Noise Optimization},
  author={Tang, Zhiwei and Peng, Jiangweizhi and Tang, Jiasheng and Hong, Mingyi and Wang, Fan and Chang, Tsung-Hui},
  journal={arXiv preprint arXiv:2405.18881},
  year={2024}
}

@InProceedings{fantasia,
    author    = {Chen, Rui and Chen, Yongwei and Jiao, Ningxin and Jia, Kui},
    title     = {Fantasia3D: Disentangling Geometry and Appearance for High-quality Text-to-3D Content Creation},
    booktitle = {Proceedings of the IEEE/CVF International Conference on Computer Vision (ICCV)},
    month     = {October},
    year      = {2023},
    pages     = {22246-22256}
}

@InProceedings{magic3d,
    author    = {Lin, Chen-Hsuan and Gao, Jun and Tang, Luming and Takikawa, Towaki and Zeng, Xiaohui and Huang, Xun and Kreis, Karsten and Fidler, Sanja and Liu, Ming-Yu and Lin, Tsung-Yi},
    title     = {Magic3D: High-Resolution Text-to-3D Content Creation},
    booktitle = {Proceedings of the IEEE/CVF Conference on Computer Vision and Pattern Recognition (CVPR)},
    month     = {June},
    year      = {2023},
    pages     = {300-309}
}

@inproceedings{tet,
 author = {Gu, Chun and Yang, Zeyu and Pan, Zijie and Zhu, Xiatian and Zhang, Li},
 booktitle = {Advances in Neural Information Processing Systems},
 pages = {80165--80190},
 publisher = {Curran Associates, Inc.},
 title = {Tetrahedron Splatting for 3D Generation},
 volume = {37},
 year = {2024}
}

@InProceedings{3dgs1,
    author    = {Chen, Zilong and Wang, Feng and Wang, Yikai and Liu, Huaping},
    title     = {Text-to-3D using Gaussian Splatting},
    booktitle = {Proceedings of the IEEE/CVF Conference on Computer Vision and Pattern Recognition (CVPR)},
    month     = {June},
    year      = {2024},
    pages     = {21401-21412}
}

@InProceedings{DreamPropeller,
    author    = {Zhou, Linqi and Shih, Andy and Meng, Chenlin and Ermon, Stefano},
    title     = {DreamPropeller: Supercharge Text-to-3D Generation with Parallel Sampling},
    booktitle = {Proceedings of the IEEE/CVF Conference on Computer Vision and Pattern Recognition (CVPR)},
    month     = {June},
    year      = {2024},
    pages     = {4610-4619}
}

@article{lrm,
  title={Lrm: Large reconstruction model for single image to 3d},
  author={Hong, Yicong and Zhang, Kai and Gu, Jiuxiang and Bi, Sai and Zhou, Yang and Liu, Difan and Liu, Feng and Sunkavalli, Kalyan and Bui, Trung and Tan, Hao},
  journal={arXiv preprint arXiv:2311.04400},
  year={2023}
}

@InProceedings{lgm,
author="Tang, Jiaxiang
and Chen, Zhaoxi
and Chen, Xiaokang
and Wang, Tengfei
and Zeng, Gang
and Liu, Ziwei",
title="LGM: Large Multi-view Gaussian Model for High-Resolution 3D Content Creation",
booktitle="Computer Vision -- ECCV 2024",
year="2025",
publisher="Springer Nature Switzerland",
address="Cham",
pages="1--18",
isbn="978-3-031-73235-5"
}

@inproceedings{integral,
  title={How I Warped Your Noise: a Temporally-Correlated Noise Prior for Diffusion Models},
  author={Chang, Pascal and Tang, Jingwei and Gross, Markus and Azevedo, Vinicius C},
  booktitle={International Conference on Learning Representations},
year="2025"
}

@InProceedings{anything,
    author    = {Yang, Lihe and Kang, Bingyi and Huang, Zilong and Xu, Xiaogang and Feng, Jiashi and Zhao, Hengshuang},
    title     = {Depth Anything: Unleashing the Power of Large-Scale Unlabeled Data},
    booktitle = {Proceedings of the IEEE/CVF Conference on Computer Vision and Pattern Recognition (CVPR)},
    month     = {June},
    year      = {2024},
    pages     = {10371-10381}
}

@article{zero,
  title={Zero123++: a single image to consistent multi-view diffusion base model},
  author={Shi, Ruoxi and Chen, Hansheng and Zhang, Zhuoyang and Liu, Minghua and Xu, Chao and Wei, Xinyue and Chen, Linghao and Zeng, Chong and Su, Hao},
  journal={arXiv preprint arXiv:2310.15110},
  year={2023}
}

@inproceedings{DreamSim,
author = {Fu, Stephanie and Tamir, Netanel Y. and Sundaram, Shobhita and Chai, Lucy and Zhang, Richard and Dekel, Tali and Isola, Phillip},
title = {DreamSim: learning new dimensions of human visual similarity using synthetic data},
year = {2023},
publisher = {Curran Associates Inc.},
address = {Red Hook, NY, USA},
booktitle = {Advances in Neural Information Processing Systems},
articleno = {2208},
numpages = {27},
location = {New Orleans, LA, USA},
}

@inproceedings{strike,
 author = {Roeder, Geoffrey and Wu, Yuhuai and Duvenaud, David K},
 booktitle = {Advances in Neural Information Processing Systems},
 pages = {},
 publisher = {Curran Associates, Inc.},
 title = {Sticking the Landing: Simple, Lower-Variance Gradient Estimators for Variational Inference},
 volume = {30},
 year = {2017}
}

@InProceedings{diverse,
author="Tran, Uy Dieu
and Luu, Minh
and Nguyen, Phong Ha
and Nguyen, Khoi
and Hua, Binh-Son",
title="Diverse Text-to-3D Synthesis with Augmented Text Embedding",
booktitle="Computer Vision -- ECCV 2024",
year="2025",
publisher="Springer Nature Switzerland",
address="Cham",
pages="217--235",
isbn="978-3-031-73226-3"
}

@InProceedings{lods,
author="Yang, Xiaofeng
and Chen, Yiwen
and Chen, Cheng
and Zhang, Chi
and Xu, Yi
and Yang, Xulei
and Liu, Fayao
and Lin, Guosheng",
title="Learn to Optimize Denoising Scores: A Unified and Improved Diffusion Prior for 3D Generation",
booktitle="Computer Vision -- ECCV 2024",
year="2025",
publisher="Springer Nature Switzerland",
address="Cham",
pages="136--152",
isbn="978-3-031-72784-9"
}

@inproceedings{hifa,
  title={HIFA: High-fidelity Text-to-3D Generation with Advanced Diffusion Guidance},
  author={Zhu, Junzhe and Zhuang, Peiye and Koyejo, Sanmi},
  booktitle={International Conference on Learning Representations},
year="2024",
}

@inproceedings{tu2024motioneditor,
  title={Motioneditor: Editing video motion via content-aware diffusion},
  author={Tu, Shuyuan and Dai, Qi and Cheng, Zhi-Qi and Hu, Han and Han, Xintong and Wu, Zuxuan and Jiang, Yu-Gang},
  booktitle={Proceedings of the IEEE/CVF Conference on Computer Vision and Pattern Recognition},
  pages={7882--7891},
  year={2024}
}

@article{tu2024motionfollower,
  title={Motionfollower: Editing video motion via lightweight score-guided diffusion},
  author={Tu, Shuyuan and Dai, Qi and Zhang, Zihao and Xie, Sicheng and Cheng, Zhi-Qi and Luo, Chong and Han, Xintong and Wu, Zuxuan and Jiang, Yu-Gang},
  journal={arXiv preprint arXiv:2405.20325},
  year={2024}
}

@inproceedings{tu2025stableanimator,
  title={Stableanimator: High-quality identity-preserving human image animation},
  author={Tu, Shuyuan and Xing, Zhen and Han, Xintong and Cheng, Zhi-Qi and Dai, Qi and Luo, Chong and Wu, Zuxuan},
  booktitle={Proceedings of the Computer Vision and Pattern Recognition Conference},
  pages={21096--21106},
  year={2025}
}

@article{tu2025stableanimator++,
  title={Stableanimator++: Overcoming pose misalignment and face distortion for human image animation},
  author={Tu, Shuyuan and Xing, Zhen and Han, Xintong and Cheng, Zhi-Qi and Dai, Qi and Luo, Chong and Wu, Zuxuan and Jiang, Yu-Gang},
  journal={arXiv preprint arXiv:2507.15064},
  year={2025}
}

@article{tu2025stableavatar,
  title={Stableavatar: Infinite-length audio-driven avatar video generation},
  author={Tu, Shuyuan and Pan, Yueming and Huang, Yinming and Han, Xintong and Xing, Zhen and Dai, Qi and Luo, Chong and Wu, Zuxuan and Jiang, Yu-Gang},
  journal={arXiv preprint arXiv:2508.08248},
  year={2025}
}

@article{kwak2024geometry,
  title={Geometry-Aware Score Distillation via 3D Consistent Noising and Gradient Consistency Modeling},
  author={Kwak, Min-Seop and Ahn, Donghoon and Kim, In{\`e}s Hyeonsu and Kim, Jin-Hwa and Kim, Seungryong},
  journal={arXiv preprint arXiv:2406.16695},
  year={2024}
}

@inproceedings{asd,
  title={Scaledreamer: Scalable text-to-3d synthesis with asynchronous score distillation},
  author={Ma, Zhiyuan and Wei, Yuxiang and Zhang, Yabin and Zhu, Xiangyu and Lei, Zhen and Zhang, Lei},
  booktitle="Computer Vision -- ECCV 2024",
year="2025",
publisher="Springer Nature Switzerland",
address="Cham",
  pages={1--19},
}
